\documentclass{article}

     \usepackage[preprint]{neurips_2022}

\usepackage[utf8]{inputenc} 
\usepackage[T1]{fontenc}    
\usepackage{hyperref}       
\usepackage{url}            
\usepackage{booktabs}       
\usepackage{amsfonts}       
\usepackage{nicefrac}       
\usepackage{microtype}      
\usepackage{xcolor}         

\usepackage{amsmath}
\usepackage{amssymb}
\usepackage{amsthm}
\usepackage{enumitem}
\usepackage{graphicx}
\usepackage{wrapfig}

\newtheorem{defn}{Definition}
\newtheorem{remark}[defn]{Remark}
\newtheorem{prop}[defn]{Proposition}
\newtheorem{lemma}[defn]{Lemma}
\newtheorem{thm}[defn]{Theorem}
\newtheorem{cor}[defn]{Corollary}

\newcommand{\myref}[2]{\ref{#1}~\ref{#1:#2}}
\newcommand{\lmat}{\begin{bmatrix}}
\newcommand{\rmat}{\end{bmatrix}}
\newcommand{\NN}{\mathbb{N}}
\newcommand{\RR}{\mathbb{R}}
\newcommand{\ZZ}{\mathbb{Z}}
\newcommand{\diag}{\operatorname{diag}}
\newcommand{\dom}{\operatorname{dom}}
\newcommand{\ran}{\operatorname{ran}}
\newcommand{\relu}{\operatorname{ReLU}}

\newcommand{\SNN}{\mathrm{SNN}}
\newcommand{\SNNirr}{\mathrm{SNN}_{\mathrm{irr}}}
\newcommand{\st}{\operatorname{st}}
\newcommand{\PZ}{\mathrm{PZ}}
\newcommand{\brak}[1]{\langle #1\rangle}

\makeatletter
\newif\ifpreprint
\@ifpackagewith{neurips_2022}{preprint}{\preprinttrue}{\preprintfalse}
\makeatother

\title{A Classification of \texorpdfstring{$G$}{G}-Invariant Shallow Neural Networks}

\author{%
Devanshu Agrawal \& James Ostrowski \\
Department of Industrial and Systems Engineering \\
University of Tennessee \\
Knoxville, TN 37996 \\
\texttt{dagrawa2@vols.utk.edu, jostrows@utk.edu} \\
}

\begin{document}

\maketitle

\begin{abstract}
When trying to fit a deep neural network (DNN) to a $G$-invariant target function with $G$ a group, it only makes sense to constrain the DNN to be $G$-invariant as well. 
However, there can be many different ways to do this, thus raising the problem of ``$G$-invariant neural architecture design'': 
What is the optimal $G$-invariant architecture for a given problem? 
Before we can consider the optimization problem itself, we must understand the search space, the architectures in it, and how they relate to one another. 
In this paper, we take a first step towards this goal; 
we prove a theorem that gives a classification of all $G$-invariant single-hidden-layer or ``shallow'' neural network ($G$-SNN) architectures with ReLU activation for any finite orthogonal group $G$, 
and we prove a second theorem that characterizes the inclusion maps or ``network morphisms'' between the architectures that can be leveraged during neural architecture search (NAS). 
The proof is based on a correspondence of every $G$-SNN to a signed permutation representation of $G$ acting on the hidden neurons; 
the classification is equivalently given in terms of the first cohomology classes of $G$, thus admitting a topological interpretation. 
The $G$-SNN architectures corresponding to nontrivial cohomology classes have, to our knowledge, never been explicitly identified in the literature previously. 
Using a code implementation, we enumerate the $G$-SNN architectures for some example groups $G$ and visualize their structure. 
Finally, we prove that architectures corresponding to inequivalent cohomology classes coincide in function space only when their weight matrices are zero, and we discuss the implications of this for NAS.
\end{abstract}

\section{Introduction}
\label{sec:intro}

When trying to fit a deep neural network (DNN) to a target function that is known to be $G$-invariant with respect to a group $G$, it is desirable to enforce $G$-invariance on the DNN as prior knowledge. 
This is a common scenario in many applications such as computer vision, where the class of an object in an image may be independent of its orientation~\citep{veeling2018rotation}, or point clouds that are permutation-invariant~\citep{qi2017pointnet}. 
Numerous $G$-invariant DNN architectures have been proposed over the years, including $G$-equivariant convolutional neural networks ($G$-CNNs)~\citep{cohen2016group}, $G$-equivariant graph neural networks~\citep{maron2019invariant}, and a DNN stacked on a $G$-invariant sum-product layer~\citep{kicki2020computationally}. 
However, it is unclear which of these architectures a practitioner should choose for a given problem, and even after one is selected, additional design choices must be made; 
for $G$-CNNs alone, the practitioner must select a sequence of representations of $G$ to determine the composition of layers, and it is unknown how best to do this. 
Moreover, despite a complete classification of $G$-CNNs~\citep{kondor2018generalization, cohen2019general}, it is unknown if every $G$-invariant DNN is a $G$-CNN, 
and hence the ``optimal'' $G$-invariant architecture may not even exist in the space of $G$-CNNs. 

For some architectures, universality theorems exist guaranteeing the approximation of any $G$-invariant function with arbitrarily small error~\citep{maron2019universality, ravanbakhsh2020universal, kicki2020computationally}, 
and it is thus tempting to conclude that these universal architectures are sufficient for all $G$-invariant problems. 
However, it is well-known that universality~\citep{cybenko1989approximations} alone is not a sufficient condition for a good DNN model and that the function subspaces that a network traverses as it grows to the universality limit is just as important as the limit itself. 
This suggests that the way in which a DNN is constrained to be $G$-invariant does matter, and different $G$-invariant architectures may be suitable for different problems. 
This raises the fundamental question: 
For a given problem, what is the ``best'' way to constrain the parameters of a DNN such that it is $G$-invariant?

This paper takes a first step towards answering the above question. 
Specifically, before we can consider the optimization problem for the best $G$-invariant architecture, we must understand the search space: 
What are all the possible ways to constrain the parameters of a DNN such that it is $G$-invariant, and how are these different $G$-invariant architectures related to one another? 

The above is a special case of the broader and more fundamental problem of neural architecture design. 
One of the most prominent approaches to this problem in the literature is neural architecture search (NAS), which at its core is trial-and-error~\citep{elsken2019neural}. 
While trial-and-error is---in principle---straightforward for determining, e.g., the optimal depth or hidden widths of a DNN, it is less clear for $G$-invariant architectures, where a practitioner does not even know all their options. 
More generally, NAS presupposes knowledge about which architectures are in the search space, which ones are not, which ones are equivalent or special cases of others, and how best one should move from one architecture to another. 
Thus, to apply even the simplest approach to $G$-invariant neural architecture design, we must first be able to enumerate all $G$-invariant architectures.

Our main result is Thm.~\ref{thm:main}, which gives a classification of all $G$-invariant single-hidden-layer or ``shallow'' neural network ($G$-SNN) architectures with rectified linear unit (reLU) activation, for any finite orthogonal group $G$ acting on the input space. 
More precisely, every $G$-SNN architecture can be decomposed into a sum of ``irreducible'' ones, and Thm.~\ref{thm:main} classifies these. 
The classification is based on a correspondence of each irreducible architecture to a representation of G in terms of its action on the hidden neurons via so-called ``signed permutations'', 
where the representation is required to satisfy an additional condition to eliminate degenerate (linear) architectures and redundant architectures equivalent to simpler ones. 
The classification then boils down to the classification of these representations. 
These representations, and hence the corresponding architectures as well, are classified in terms of the first cohomology classes of $G$ and thus admit a topological interpretation. 
We note that, while connections between neural networks and the group of signed permutations have been previously made in the literature~\citep{ojha2000enumeration, negrinho2014orbit, arjevani2020analytic}, to our knowledge, no such connection has yet been leveraged to begin a classification program of $G$-invariant architectures.

We also prove Thm.~\ref{thm:ai}, which characterizes the ``network morphisms'' linking irreducible $G$-SNN architectures in architecture space. 
In NAS, network morphisms furnish a topology on architecture space and describe how one should move from one architecture to another during the search~\citep{wei2016network}. 
Taken together, Thms.~\ref{thm:main}\&\ref{thm:ai} give a complete description of $G$-SNN architecture space.

This paper is perhaps most similar in spirit to the works of \citet{kondor2018generalization} and \citet{cohen2019general} and draws on similar mathematical machinery; 
like them, this paper's contribution is also primarily theoretical. 
\citet{kondor2018generalization} prove that every $G$-invariant DNN is a $G$-CNN under the assumption that every affine layer is $G$-equivariant; 
that this is true without the assumption is only conjectured. 
\citet{cohen2019general} generalize this to $G$-CNNs where hidden activations are vector fields and provide a classification of all $G$-CNNs, 
but the conjecture of \citet{kondor2018generalization} is left open. 
In contrast to these works, in our paper, we do not assume  the pre-activation affine transformation to be $G$-equivariant-- only that the whole network is $G$-invariant. 
Thus, a future extension of Thm.~\ref{thm:main} to deep architectures would either prove or refute the cited conjecture, at least for ReLU networks. 
Moreover, these works do not explicitly work out the group representations compatible with ReLU, 
and other works~\citep{cohen2016group, cohen2019gauge} consider only unsigned permutations with ReLU. 
In contrast, our classification reveals $G$-SNN architectures (namely, those corresponding to proper signed permutation representations or nontrivial cohomology classes) that, to our knowledge, have never been explicitly identified in the literature previously.

We also note the work of \citet{maron2019invariant}, who classify all $G$-equivariant linear layers for graph neural networks; 
however, they restrict their attention to unsigned permutation representations only, 
and there classification is again not guaranteed to contain all $G$-invariant ReLU networks.

The remainder of the paper is organized as follows: 
In Sec.~\ref{sec:rho}, we give a classification of the ``signed permutation representations'' of $G$ and relate these representations to the cohomology classes of $G$.%
\footnote{We assume some familiarity with group theory including semidirect products and quotient groups, conjugacy classes of subgroups, and group action~\citep[see][]{herstein2006topics}.} %
Then in Sec.~\ref{sec:gsnn}, we build towards and state our main classification theorem of $G$-SNN architectures. 
While Sec.~\ref{sec:rho} makes little reference to $G$-SNNs, presenting it upfront helps to streamline the exposition in Sec.~\ref{sec:gsnn}, with much of the notation and terminology established. 
In Sec.~\ref{sec:examples}, we visualize the $G$-SNN architectures for some example groups $G$, 
and in Sec.~\ref{sec:remarks}, we make a number of remarks including a theorem on the ``network morphisms'' between $G$-SNN architectures. 
Finally, in Sec.~\ref{sec:conclusion}, we end with conclusions and next steps towards the problem of $G$-invariant neural architecture design.

\section{Signed permutation representations}
\label{sec:rho}

\subsection{Preliminaries}

Throughout this paper, let $G$ be a finite group of $m\times m$ orthogonal matrices. 
Let $\mathcal{P}(n)$ be the group of all $n\times n$ permutation matrices 
and $\mathcal{Z}(n)$ the group of all $n\times n$ diagonal matrices with diagonal entries $\pm 1$. 
Let $\PZ(n) = \mathcal{P}(n)\ltimes\mathcal{Z}(n)$, which is the group of all \emph{signed permutations}-- i.e., 
the group of all permutations and reflections of the standard orthonormal basis $\{e_1,\ldots,e_n\}$. 
This group is also called the hyperoctahedral group in the literature~\citep{baake1984structure}.

A \emph{signed permutation representation} (signed perm-rep) of degree $n$ of $G$ is a homomorphism $\rho:G\mapsto\PZ(n)$. 
Whenever we say $\rho$ is a signed perm-rep, let it be understood that its degree is $n$ unless we say otherwise. 
A signed perm-rep $\rho$ is said to be \emph{irreducible} if for every $i,j=1,\ldots,n$, there exists $g\in G$ such that $\rho(g)e_i = \pm e_j$. 
As we will see in Sec.~\ref{sec:gsnn2rho}, every $G$-SNN can be written as a sum of ``irreducible'' $G$-SNNs, 
and every irreducible $G$-SNN corresponds to an irreducible signed perm-rep. 
It is therefore sufficient for our purposes to classify all irreducible signed perm-reps of $G$; 
moreover, this need only be done up to conjugacy as seen next.

\subsection{Classification up to conjugacy}

Two signed perm-reps $\rho,\rho^{\prime}$ are said to be \emph{conjugate} if there exists $A\in\PZ(n)$ such that $\rho^{\prime}(g) = A^{-1}\rho(g)A\forall g\in G$. 
We let $\rho^{\PZ}$ denote the conjugacy class of the signed perm-rep $\rho$. 
Note that conjugation preserves the (ir)reducibility of a signed perm-rep (Prop.~\ref{prop:irrconj} in Supp.~\ref{appendix:rho_clf_conj}), 
and it thus makes sense to speak of the irreducibility of an entire conjugacy class $\rho^{\PZ}$. 
The significance of the conjugacy relation is that conjugate signed perm-reps correspond to the same $G$-SNN (see Sec.~\ref{sec:gsnn2rho}); 
we are thus interested in the classification of irreducible signed perm-reps only up to conjugacy.

Our first theorem below gives the desired classification of signed perm-reps.%
\footnote{All proofs, as well as additional lemmas and useful propositions, can be found in the supplementary material.} %
For $H,K\leq G$, let $(H, K)^G$ denote the \emph{paired conjugacy class}
\[ (H, K)^G = \{(g^{-1}Hg, g^{-1}Kg): g\in G\}. \]
Define the following set of conjugacy classes of subgroup pairs:
\[ \mathcal{C}^G_{\leq 2} = \{(H, K)^G: K\leq H\leq G\mid |H:K|\leq 2\}. \]

For every $(H, K)^G\in \mathcal{C}^G_{\leq 2}$, we define a signed perm-rep $\rho_{HK}$ as follows: 
Let $(g_1,\ldots,g_n)$ be a transversal of $G/H$ with $g_1\in K$. 
For each $i=1,\ldots,n$, define $g_{-i} = g_i h$ for some $h\in H\setminus K$ if $|H:K|=2$ and $h=1$ if $|H:K|=1$. 
Then define the signed perm-rep $\rho_{HK}$ such that $\rho_{HK}(g)e_i = e_j$ if $gg_i K = g_j K$ for every $i,j=\pm 1,\ldots,\pm n$. 

\begin{thm} \label{thm:rho_clf_conj}
We have:
\begin{enumerate}[label={(\alph*)}]
\item \label{thm:rho_clf_conj:a} 
Every $\rho_{HK}$ is irreducible.
\item \label{thm:rho_clf_conj:b} 
For every irreducible signed perm-rep $\rho^{\prime}$, there exists a unique $(H, K)^G$ such that $\rho^{\prime}$ is conjugate to $\rho_{HK}$.
\end{enumerate}
\end{thm}

Theorem~\ref{thm:rho_clf_conj} equivalently states that the set
\[ \mathcal{R}^{\PZ}(G) = \{\rho_{HK}^{\PZ}: (H, K)^G\in\mathcal{C}^G_{\leq 2}\} \]
is a partition on the set of all irreducible signed perm-reps into conjugacy classes. 
We will say $\rho_{HK}$ has \emph{type} $|H:K|$-- 
i.e., type 1 if $H=K$ and type 2 otherwise. 
For the interested reader, we note that $\rho_{HK}$ is the rep of $G$ induced from the rep $\phi:H\mapsto \{-1, 1\}$ where $K = \ker(\phi)$.

\subsection{Group cohomology}
\label{sec:cohomology}

Group cohomology offers an alternative perspective on the classification of signed perm-reps 
and is the basis for the directed graph visualizations in Sec.~\ref{sec:examples:perm} and Supp.~\ref{appendix:examples:perm}.%
\footnote{These visualizations are based on a geometric perspective of cohomology on Cayley graphs~\citep[sec. 5.9]{dructu2018geometric}; see \citep{tao2012cayley} for intuition.}
Note, however, that a technical understanding of group cohomology is not required for most of this paper, 
and we give only a high-level overview here. 
Let $\rho$ be a signed perm-rep 
and $\pi:G\mapsto\mathcal{P}(n)$ and $\omega:G\mapsto\mathcal{Z}(n)$ the unique functions satisfying $\rho(g) = \omega(g)\pi(g)\forall g\in G$. 
The function $\omega$ is called a \emph{cocycle} and describes the sign flips associated to the action of $G$ through $\rho$. 
It can be depicted using colored directed graphs as in Fig.~\ref{fig:C6_perm} (and Figs.~\ref{fig:D6_perm:1}-\ref{fig:D6_perm:3}), 
where arcs of different colors represent the actions of diffrent generators of $G$ 
and dashed arcs represent sign flips. 
The cocycle $\omega$ thus encodes topological information about how a space ``twists'' as we move through it by the action of $G$. 

If $\rho^{\prime}$ is another signed perm-rep conjugate to $\rho$ by a diagonal matrix in $\mathcal{Z}(n)$, then its corresponding cocycle $\omega^{\prime}$ is said to be \emph{cohomologous} to $\omega$, 
and the set of all cocycles cohomologous to $\omega$ is said to form a \emph{cohomology class}. 
the reason for this equivalence is that the number of sign flips around a cycle is unique only up to an even number of sign flips; 
thus, the solid vs. dashed arcs of the directed graphs in Fig.~\ref{fig:C6_perm} are not unique. 
The set of all cohomology classes associated to a single $\pi$ forms a \emph{cohomology group}, 
and it turns out that the classification of these cohomology groups and their elements gives another way to look at the classification of the signed perm-reps of $G$; 
see Prop.~\ref{prop:cohomology} in Supp.~\ref{appendix:cohomology} for details. 
Type 1 signed perm-reps are then the ones corresponding to the identity elements of these cohomology groups---i.e., the cohomology classes with no sign flips---%
and type 2 signed perm-reps correspond to the elements describing nontrivial sign flip patterns.

Finally, as we are only interested in signed perm-reps up to conjugation, we regard two cohomology classes as equivalent if their corresponding signed perm-reps are related by conjugation with a permutation matrix in $\mathcal{P}(n)$.%
\footnote{For the interested reader, this equivalence amounts to quotienting the cohomology group $\mathcal{H}^1(G, M)$ by the automorphism group of the $G$-module $(M, \pi)$; 
see Prop.~\myref{prop:cohomology}{b} in Sup.~\ref{appendix:cohomology} for details.} %
Concretely, this ensures that colored directed graphs isomorphic to those shown in Fig.~\ref{fig:C6_perm} are in fact considered equivalent to them.

\section{Classification of \texorpdfstring{$G$}{G}-SNNs}
\label{sec:gsnn}

\subsection{Canonical parameterization}
\label{sec:canonical}

A \emph{shallow neural network} (SNN) is a function $f:\RR^m\mapsto\RR$ of the form
\begin{equation} \label{eq:snn}
f(x) = a^\top\relu(Wx+b)+d,
\end{equation}
where $W\in\RR^{n\times m}$ and $a,b\in\RR^n$ for some $n$, and $d\in\RR$. 
Here $\relu$ is the \emph{rectified linear unit} activation function defined as $\relu(x) = \operatorname{max}(0, x)$ elementwise. 
The parameterization of an SNN given in Eq.~\ref{eq:snn} contains redundancies in the sense that different parameter configurations can define the same function. 
For example, applying a permutation to the rows of $a$, $b$, and $W$ generally results in a different parameter configuration but always leaves $f$ invariant. 
Also, by the identity
\begin{equation} \label{eq:relulinear}
\relu(zx) = \relu(x) - H(-z)x, x\in\RR, z\in\{-1, 1\},
\end{equation}
(where $H$ is the Heaviside step function; see Prop.~\ref{prop:relulinear} in Supp.~\ref{appendix:gsnn}), 
the reflection of one or more rows of the augmented matrix $[W\mid b]$ in Eq.~\ref{eq:snn} together with the addition of a linear term---which can be represented as the sum of two hidden neurons---leaves $f$ invariant; 
e.g., 
\[ f(x) = a^\top\relu[-(Wx+b)] + \relu[a^\top(Wx+b)] - \relu[-a^\top(Wx+b)] + d. \]
Note that these permutation and reflection redundancies form the group $\PZ(n)$. 
The lemma below will help us define a ``canonical parameterization'' in which such redundancies are eliminated. 

\begin{lemma} \label{lemma:canonical}
Let $\Theta_n$ be the set of all augmented matrices $[W\mid b]\in\RR^{n\times (m+1)}$ such that the rows of $W$ have unit norms and no two rows of $[W\mid b]$ are parallel. 
Let $\Omega_n\subset\Theta_n$ be a fundamental domain%
\footnote{If a group $G$ acts on a set $\mathcal{X}$, then a fundamental domain is a set $\Omega\subseteq\mathcal{X}$ such that $\{g\Omega: g\in G\}$ is a partition of $\mathcal{X}$. 
An example fundamental domain in $\Theta_n$ under the action of $\PZ(n)$ is given in Prop.~\ref{prop:domain} (see Supp.~\ref{appendix:canonical}).} %
under the action of $\PZ(n)$. 
Let $f:\RR^m\mapsto\RR$ be an SNN of the form in Eq.~\ref{eq:snn}. 
Then there exist unique $n_*\in\NN$, $[W_*\mid b_*]\in\Omega_{n_*}$, $a_*\in\RR^{n_*}$ with nonzero elements, $c_*\in\RR^m$, and $d_*\in\RR$ such that:
\[ f(x) = a_*^\top\relu(W_*x+b_*) + c_*^\top x+d_*\forall x\in\RR^m. \]
\end{lemma}

Here, since the rows of $[W_*\mid b_*]$ are pairwise nonparallel, then no two hidden neurons can form a linear term; 
all such hidden neurons are collected in the unique term $c_*^\top x$. 
As a result, $n_*\leq n$ as $n_*$ is the smallest number of hidden neurons possible. 
We refer to the unique parameterization of the SNN $f$ in terms of $(a_*, b_*, c_*, d_*, W_*)$ as its \emph{canonical parameterization}, 
and the SNN is then said to be in \emph{canonical form}. 
We call $b_*$, $W_*$, and the rows of $W_*$ the canonical \emph{bias}, \emph{weight matrix}, and \emph{weight vectors} of the $G$-SNN respectively. 
Note that the canonical parameterization is a function of the choice of fundamental domain $\Omega_{n_*}$.

\subsection{\texorpdfstring{$G$}{G}-SNNs and signed perm-reps}
\label{sec:gsnn2rho}

For $f$ to be a $G$-invariant SNN ($G$-SNN), the action of $g\in G$ on the domain of $f$ must be equivalent to one of the redundancies in the parameterization of SNNs. 
In the canonical parameterization, however, this means that the parameters of $f$ and of $f\circ g$ must be identical. 
This places constraints on the canonical parameters of a $G$-SNN, as made precise in the lemma below. 

\begin{lemma} \label{lemma:gsnn2rho}
Let $f:\RR^m\mapsto\RR$ be an SNN expressed in canonical form with respect to a fundamental domain $\Omega_{n_*}\subset\Theta_{n_*}$. 
Then $f$ is $G$-invariant if and only if there exists a unique signed perm-rep $\rho:G\mapsto\PZ(n_*)$ such that the canonical parameters of $f$ satisfy the following equations for all $g\in G$:
\begin{align}
\rho(g)W_* &= W_*g \label{eq:gsnn2rho:W} \\
\pi(g)a_* &= a_* \label{eq:gsnn2rho:a} \\
\rho(g)b_* &= b_* \label{eq:gsnn2rho:b} \\
gc_* &= c_* + \frac{1}{2}(I-g)W_*^\top a_*. \label{eq:gsnn2rho:c}
\end{align}
\end{lemma}

We see that the constraints on the canonical parameters of a $G$-SNN are not necessarily unique, as they depend on a signed perm-rep $\rho$, whence the classification program of $G$-SNNs in this paper. 
Lemma~\ref{lemma:gsnn2rho} thus establishes the promised connection between $G$-SNNs and signed perm-reps. 
To formalize this correspondence, let $\SNN(G)$ be the set of all $G$-SNNs and $\mathcal{R}^{\PZ}(G)$ the set of all conjugacy classes of signed perm-reps of $G$. 
Define the map $F:\SNN(G)\mapsto\mathcal{R}^{\PZ}(G)$ such that if $f\in\SNN(G)$ and $\rho$ is the corresponding signed perm-rep appearing in Lemma~\ref{lemma:gsnn2rho}, then $F(f) = \rho^{\PZ}$. 
Then $F$ is a well-defined function in the sense that $F(f)$ does not depend on the choice of fundamental domain $\Omega_{n_*}$. 
Indeed, a change of fundamental domain $\Omega_{n_*}\rightarrow\Omega_{n_*}^{\prime}$ induces a transformation $[W_*\mid b_*]\rightarrow A[W_*\mid b_*]$ for a unique $A\in\PZ(n_*)$. 
By Eq.~\ref{eq:gsnn2rho:W}, this in turn induces the conjugation $\rho\rightarrow A\rho(\cdot)A^{-1}$, thereby leaving $F(f)=\rho^{\PZ}$ invariant.

We now define a \emph{$G$-SNN architecture} to be a subset $S\subseteq\SNN(G)$ such that $S=F^{-1}(\rho^{\PZ})$ for some $\rho^{\PZ}\in\ran(F)$; 
it consists of all $G$-SNNs that are constrained to respect the same representation of $G$. 
In this language, the purpose of this paper is to classify all $G$-SNN architectures.

A $G$-SNN $f$ is said to be \emph{irreducible} if $F(f)$ is a conjugacy class of irreducible signed perm-reps. 
Let $f^{\PZ}$ denote the $G$-SNN architecture containing the $G$-SNN $f$, 
and observe that if $f$ is irreducible, then so are all $G$-SNNs in $f^{\PZ}$; 
in this case, $f^{\PZ}$ is said to be an \emph{irreducible architecture}. 

It can be shown that every $G$-SNN admits a decomposition into a sum of irreducible $G$-SNNs (Prop.~\ref{prop:f_decomposition} in Supp.~\ref{appendix:gsnn2rho}). 
It follows that to classify all $G$-SNN architectures, it is enough to classify all irreducible $G$-SNN architectures. 
This amounts to two tasks: 
(1) Classify all irreducible signed perm-rep conjugacy classes in $\ran(F)$, and 
(2) for every irreducible $\rho^{\PZ}\in\ran(F)$, give a parameterization of all $G$-SNNs in the architecture $F^{-1}(\rho^{\PZ})$.

\subsection{The classification theorem}
\label{sec:thm}

We now state our main theorem, but first we introduce some notation. 
If $A$ is a linear operator (resp.~set of linear operators), then let $P_A$ be the orthogonal projection operator onto the vector subspace that is pointwise-invariant under the action of $A$ (resp.~all elements of $A$). 
Note that if $A$ is a finite orthogonal group, then~\citep[sec. 2.6]{serre1977linear}
\[ P_A = \frac{1}{|A|}\sum_{a\in A} a. \]
Let $\st_G(P_A)$ denote the stabilizer subgroup
\[ \st_G(P_A) = \{g\in G: gP_A = P_A\}. \]

\begin{thm} \label{thm:main}
Let $\rho_{HK}$ be an irreducible signed perm-rep of $G$, 
and let $\{g_1,\ldots,g_n\}$ be a transversal of $G/H$ such that $\rho_{HK}(g_i)e_1 = e_i$. 
Let $\tau = |H:K|-1$. 
Then:
\begin{enumerate}[label={(\alph*)}]
\item \label{thm:main:a} 
$\rho_{HK}^{\PZ}\in\ran(F)$ if and only if $\st_G(P_K-\tau P_H) = K$.
\item \label{thm:main:b} 
If $\rho_{HK}^{\PZ}\in\ran(F)$, then $f\in F^{-1}(\rho_{HK}^{\PZ})$ if and only if the canonical parameters%
\footnote{This is an important subtlety. 
By specifying ``canonical parameters'', we exclude from Eqs.~\ref{eq:main:W}-\ref{eq:main:c} parameter values that do not correspond to a canonical parameterization. 
For example, $w$ cannot lie in a proper subspace $\ran(P_{K^{\prime}}-\tau P_{H^{\prime}})\subset \ran(P_K-\tau P_H)$ as this would result in at least two rows of $[W_*\mid b_*]$ being parallel. 
For the same reason, for type 1 architectures, we must have $b\neq 0$ if any two rows of $W_*$ are antiparallel.%
}
of $f$ have the following forms:
\begin{align}
W_* &= \sum_{i=1}^n e_i (g_i w)^\top, w\in\ran(P_K-\tau P_H), \Vert w\Vert=1 \label{eq:main:W} \\
a_* &= a\vec{1}, a\neq 0 \label{eq:main:a} \\
b_* &= (1-\tau)b\vec{1}, b\in\RR \label{eq:main:b} \\
c_* &= -\frac{1}{2}\tau W_*^{\top}a_* + c, c\in\ran(P_G). \label{eq:main:c}
\end{align}
\end{enumerate}
\end{thm}

The condition in Thm.~\myref{thm:main}{a} helps to exclude architectures where Eq.~\ref{eq:main:W} yields redundant weight vectors. 
Combining this with the classification of irreducible signed perm-reps up to conjugacy (Thm.~\ref{thm:rho_clf_conj}), we obtain a complete classification of the irreducible $G$-SNN architectures as an immediate corollary. 
Theorems~\ref{thm:rho_clf_conj}\&\ref{thm:main} can be assembled into an algorithm that enumerates all irreducible $G$-SNN architectures for any given finite orthogonal group $G$. 
%
%
We implemented the enumeration algorithm using a combination of GAP%
\footnote{GAP is a computer algebra system for computational discrete algebra with particular emphasis on computational group theory~\citep{anon2021gap}.} 
and Python; 
our implementation currently supports, in principle, all finite permutation groups $G<\mathcal{P}(m)$.%
\footnote{Code for our implementation and for reproducing all results in this paper %
is available at: \url{https://github.com/dagrawa2/gsnn_classification_code}.
} %
Using our code implementation, we enumerated all irreducible $G$-SNN architectures for one permutation representation of every group $G$, $|G|\leq 8$, up to isomorphism. 
We report the number of architectures, broken down by type, for each group in Table~\ref{table:numbers} (see Supp.~\ref{appendix:examples:count}; a discussion is included there as well). 
A key observation is that the number of type 2 architectures---which to our knowledge have never appeared in the literature previously---is significant; 
e.g., for the dihedral permutation group $G=D_4$ on four elements, there are five type 1 architectures compared to seven type 2 architectures.

Script execution time for each group was under $2$ seconds. 
Nevertheless, we remark that our code is not optimized for speed and scalability as our purpose was exploration and intuition. 
As future work, we will work out Thm.~\ref{thm:main} for specific families of groups to derive more direct and efficient implementations. 
We will also investigate how to build $G$-SNN architectures from smaller $G_1$-SNN and $G_2$-SNN architectures where $G$ is a (semi)direct product of $G_1$ and $G_2$. 
Finally, we note that for the application of NAS, we will probably never enumerate \emph{all} $G$-SNN architectures; 
instead, we will generate them on-the-fly as we move through the search space.

\section{Examples}
\label{sec:examples}

\subsection{The cyclic permutation group}
\label{sec:examples:perm}

\begin{figure}
\centering
\begin{minipage}{\textwidth}
\begin{minipage}{0.5\textwidth}
\centering
Architecture 0.0

\includegraphics[width=0.5\textwidth]{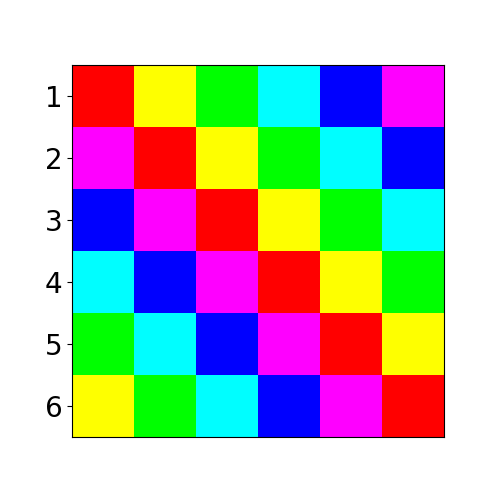}%
\includegraphics[width=0.5\textwidth]{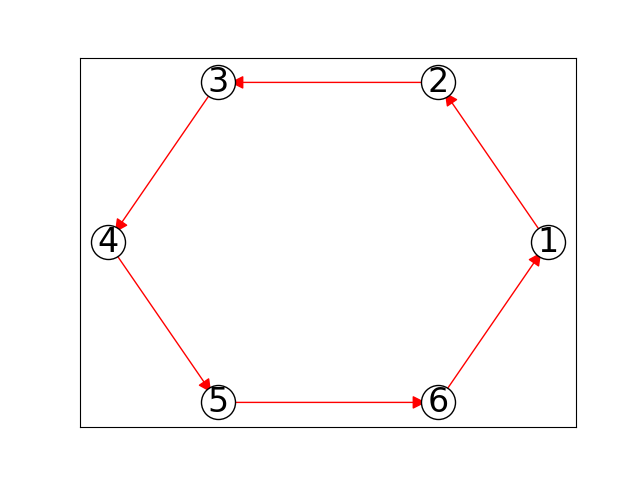}

Architecture 1.0

\includegraphics[width=0.5\textwidth]{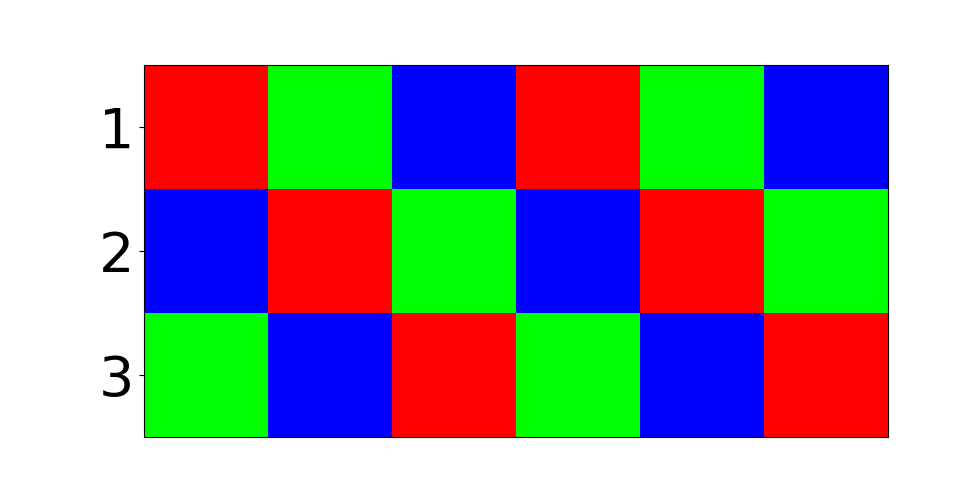}%
\includegraphics[width=0.5\textwidth]{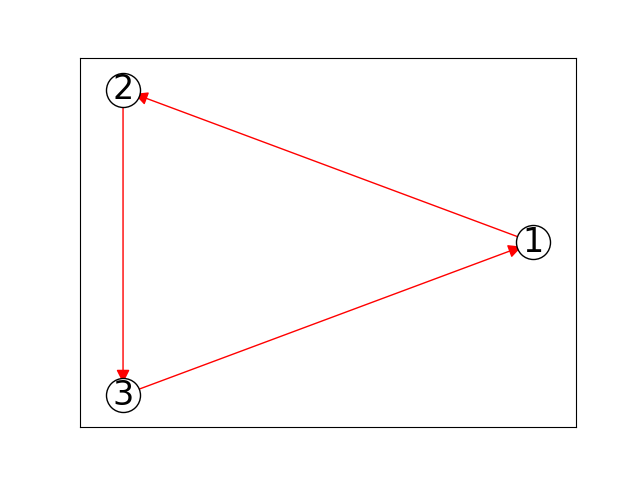}
\end{minipage}%
\begin{minipage}{0.5\textwidth}
\centering
Architecture 2.0

\includegraphics[width=0.5\textwidth]{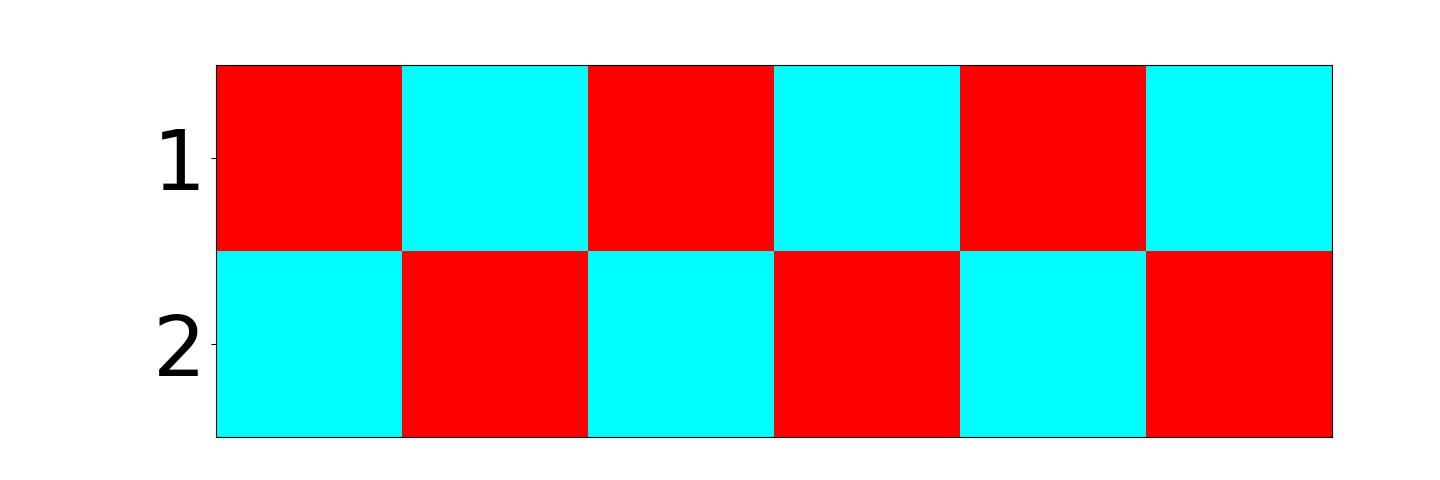}%
\includegraphics[width=0.5\textwidth]{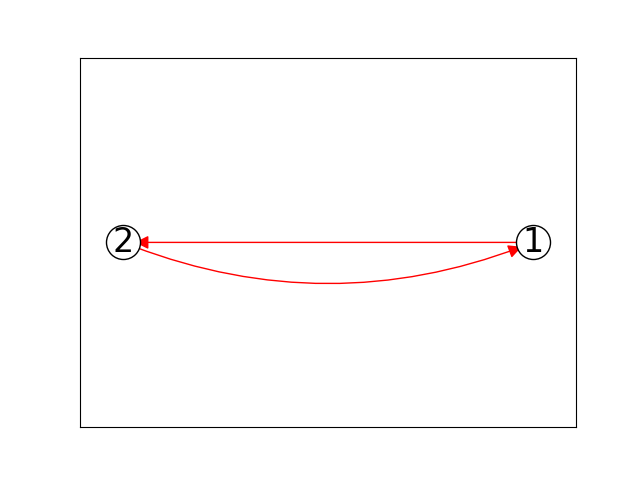}

Architecture 3.0

\includegraphics[width=0.5\textwidth]{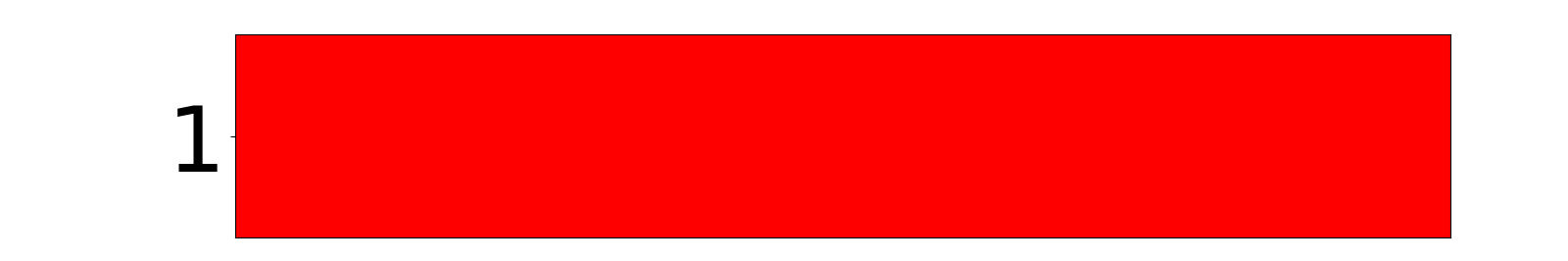}%
\includegraphics[width=0.5\textwidth]{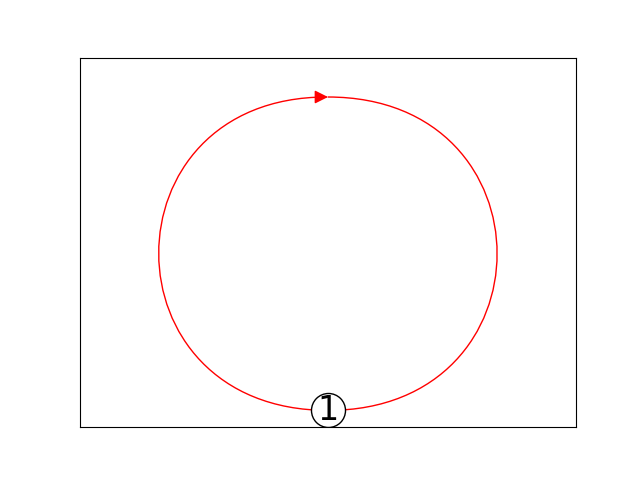}
\end{minipage}
\end{minipage}

\rule{0.9\textwidth}{0.4pt}

\begin{minipage}{\textwidth}
\begin{minipage}{0.5\textwidth}
\centering
Architecture 1.1

\includegraphics[width=0.5\textwidth]{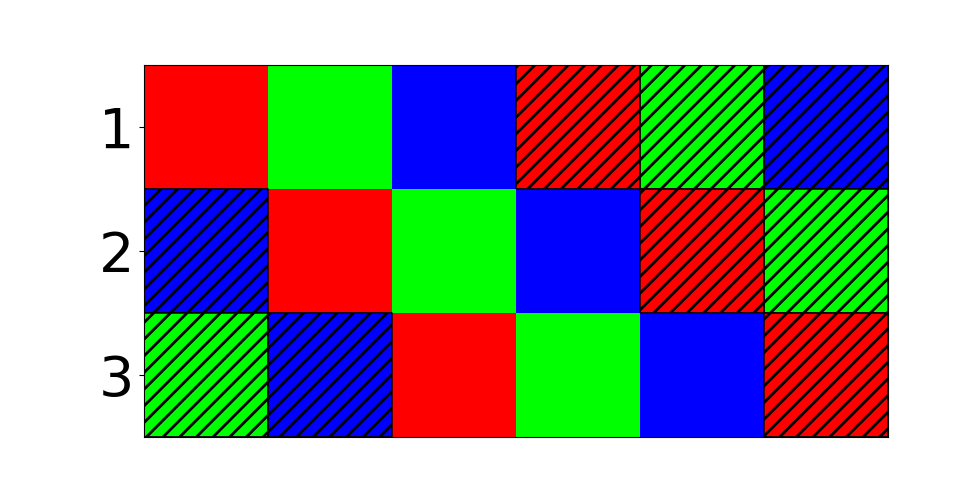}%
\includegraphics[width=0.5\textwidth]{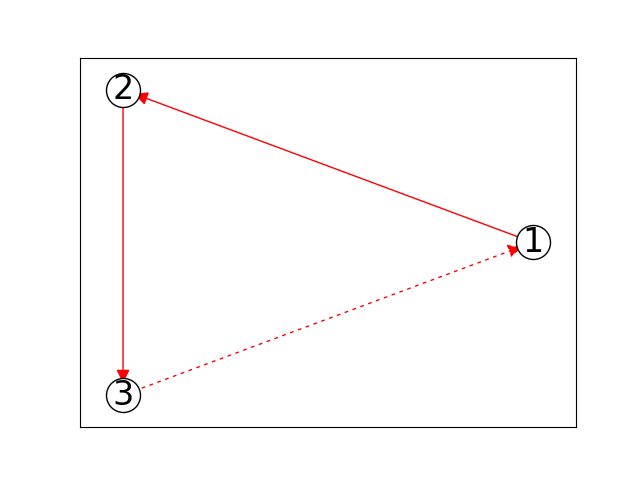}
\end{minipage}%
\begin{minipage}{0.5\textwidth}
\centering
Architecture 3.1

\includegraphics[width=0.5\textwidth]{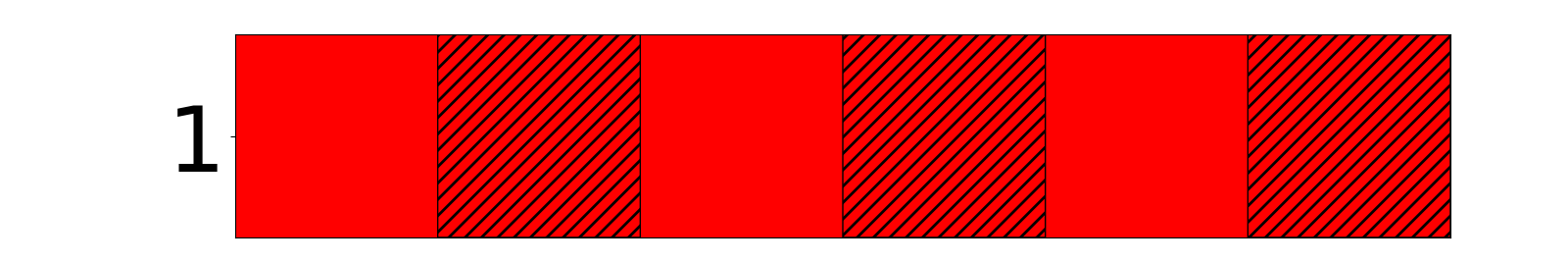}%
\includegraphics[width=0.5\textwidth]{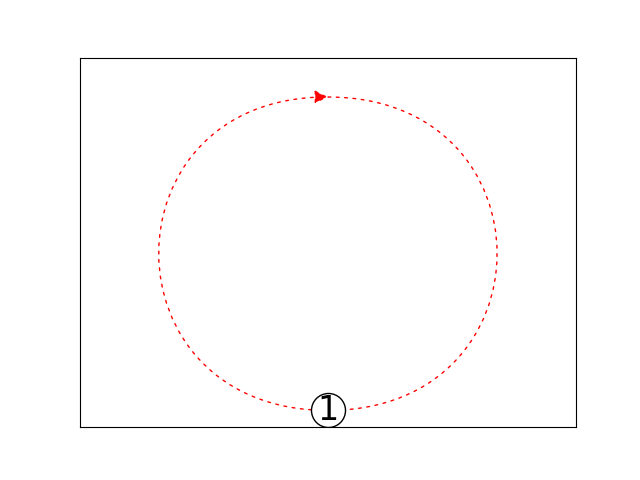}
\end{minipage}
\end{minipage}
\caption{\label{fig:C6_perm} %
Constraint pattern of the weight matrix and illustration of the cohomology class of each irreducible $G$-SNN architecture for the cyclic permutation group $G=C_6$. 
The number of rows (resp. columns) in each pattern is the number of hidden (resp. input) neurons in the architecture. 
In each pattern, weights of the same color and texture (solid vs. hatched) are constrained to be equal; 
weights of the same color but different texture are constrained to be opposites 
(colors should not be compared across different architectures). 
In each cohomology class illustration, the nodes represent the hidden neurons of the architecture, 
and the arcs represent the action of the generators of $G$ on the rows of the weight matrix 
(all arcs are the same color because $C_6$ has only one generator). 
Solid (resp. dashed) arcs preserve (resp. reverse) orientation. 
See Supp.~\ref{appendix:examples:perm} for a richer example-- the dihedral permutation group $D_6$.
}
\end{figure}

Consider the group $G=C_6$ of all cyclic permutations on the dimensions of the input space $\RR^6$.%
\footnote{See Supp.~\ref{appendix:examples} for richer examples that could not fit in the main paper.} %
There are six irreducible $G$-SNN architectures for $G=C_6$ (Fig.~\ref{fig:C6_perm}); 
``architecture i.j'' refers to $F^{-1}(\rho^{\PZ}_{H_iK_j})$ where $H_0,\ldots,H_3$ are isomorphic to $C_1$, $C_2$, $C_3$, and $C_6$ respectively and $K_j\leq H_i$ such that $|H_i:K_j|=j+1$. 
Architectures i.0 are thus exactly the type 1 ones, and two architectures i.j and i.k for distinct j and k correspond to inequivalent cohomology classes in the same cohomology group. 
Note that the architectures with $n$ hidden neurons correspond to $H_i \cong C_{\frac{6}{n}}$.

The type 1 architectures i.0 correspond to ordinary unsigned perm reps of $G$. 
These are the ``obvious'' architectures that practitioners probably could have intuited. 
From Fig.~\ref{fig:C6_perm}, we see that the weight matrices of these architectures are constrained to have a circulant structure; 
cycling the input neurons is thus equivalent to cycling the hidden neurons, leaving the output invariant as all weights in the second layer (not depicted) are constrained to be equal. 
This circulant structure is also apparent in the cohomology class illustrations. 

Architectures i.1 are type 2 and are perhaps less obvious. 
Cycling the input neurons is equivalent to cycling the hidden neurons only up to sign; 
if we cycle a weight vector around all the hidden neurons, then we do not return to the original weight vector but instead to its opposite. 
If we think of the dashed arcs in the cohomology class illustrations as ``half-twists'' in a cylindrical band, then Architectures i.1 correspond to a M\"{o}bius band, 
thereby distinguishing them topologically from architectures i.0. 
Alternatively, in terms of graph colorings, if the nodes incident to a solid (resp. dashed) arc are constrained to have the same (resp. different) color(s), then architectures i.1 are the only ones not $2$-colorable.

Observe that the top weight vector of architecture i.1 is constrained to be orthogonal to that of architecture i.0; 
this is made precise in Prop.~\ref{prop:orthogonal} in Sec.~\ref{sec:topological}. 
The upshot is that architectures i.0 and i.1 can coincide in function space if and only if their weight matrices vanish-- i.e., the architectures degenerate into linear functions. 
Since neural networks are trained with local optimization, then we think it is unlikely that a $G$-SNN being fit to a nonlinear dataset will degenerate to a linear function at any point in its training; 
assuming this is true, architectures i.0 and i.1 are effectively confined from one another due to their inequivalent topologies. 
We discuss this phenomenon in more detail in Sec.~\ref{sec:topological}.

\subsection{The cyclic rotation group}
\label{sec:examples:rot}

\begin{figure}
\centering
\begin{tabular}{cc}
\includegraphics[width=0.5\textwidth]{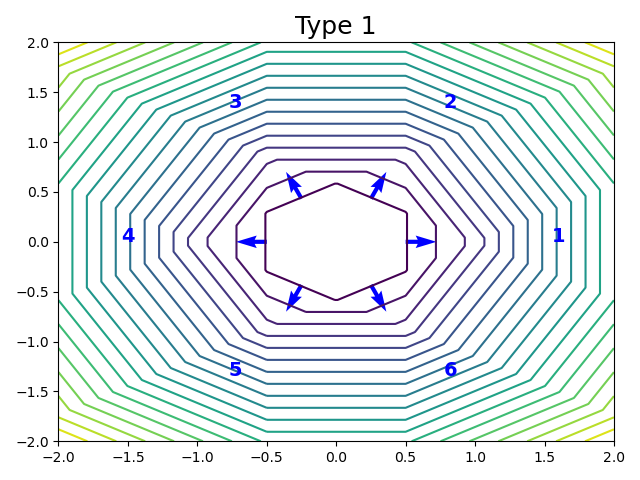} & %
\includegraphics[width=0.5\textwidth]{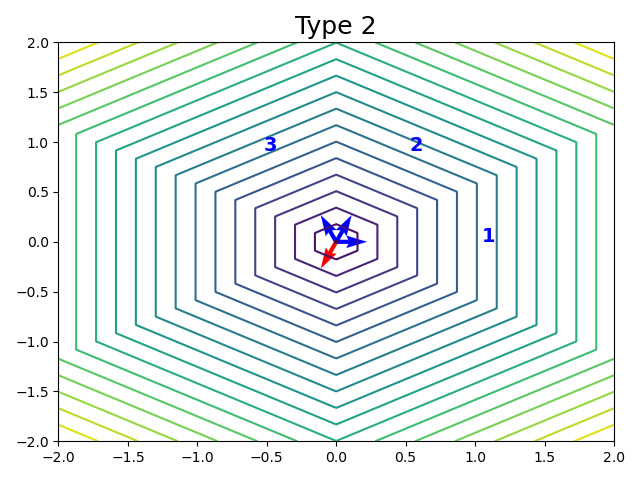}
\end{tabular}
\caption{\label{fig:C6_rot} %
Contour plots of the two irreducible $G$-SNN architectures for the 2D orthogonal representation of $G=C_6$. 
The blue vectors are the weight vectors-- i.e., rows of the weight matrix $W_*$, 
and their offsets from the origin in the type 1 architecture indicate the bias $b_*$. 
The red vector in the type 2 architecture is the canonical parameter $c_*$ of the $G$-SNN.
See Supp.~\ref{appendix:examples:rot} for a richer example-- the dihedral rotation group $D_6$.
}
\end{figure}

Consider again the group $G=C_6$, but this time a 2D orthogonal representation where each group element acts as a rotation by a multiple of $60^\circ$ on the 2D plane. 
There are only two irreducible $G$-SNN architectures-- one of each type. 
To visualize these architectures, we set $w = [1, 0]^{\top}$, $a=1$, $b=0.5$, $c=0$, and $d_*=0$ in Thm.~\myref{thm:main}{b}. 
Based on their contour plots (Fig.~\ref{fig:C6_rot}), 
we find that the level curves of the type 1 (resp. type 2) architecture are concentric regular dodecagons (resp. hexagons); 
both architectures are thus clearly invariant to $60^\circ$-rotations.

In the type 2 architecture, the hexagonal level curves increase linearly with radial distance. 
Since the bias is required to be zero (Eq.~\ref{eq:main:b}), a sharp minimum forms at the origin. 
The architecture has three weight vectors and thus three hidden neurons, 
and it additionally has a linear term (whose gradient is shown in red in Fig.~\ref{fig:C6_rot}), which---when combined with weight vector 2 using Eq.~\ref{eq:relulinear}---results in three weight vectors with $C_3$ symmetry. 
Observe that if we cycle the three hidden neurons of the type 2 architecture , so that each weight vector is rotated three times by $60^\circ$, then we obtain the three weight vectors with reversed orientation; 
this is a manifestation of the nontrivial topology of the type 2 architecture.

The type 1 architecture has six weight vectors and thus six hidden neurons. 
Observe that for each weight vector, there is another that is its opposite. 
Thus, for $[W_*\mid b_*]$ to have pairwise nonparallel rows (see Lemma~\ref{lemma:canonical}), the bias $b_* = b\vec{1}$ must be nonzero, 
whence the dodecahedral region in the example $G$-SNN (Fig.~\ref{fig:C6_rot}) where its value plateaus to zero. 
However, in the asymptotic limit $b\rightarrow 0$, the type 1 architecture degenerates to the type 2 architecture but with twice the number of hidden neurons. 
Thus, even though the two architectures are topologically distinct, the type 1 architecture can get arbitrarily close to the type 2 architecture in function space 
(see Supp.~\ref{appendix:examples:rot} for a richer example---the dihedral rotation group $D_6$---which has irreducible architectures that cannot easily access one another). 
This has important consequences, which we discuss more in the next section.

\section{Remarks}
\label{sec:remarks}

\subsection{Numbers of hidden neurons}
\label{sec:numbers}

The type 1 and type 2 irreducible architectures for the $G=C_6$ rotation group (Fig.~\ref{fig:C6_rot}) have six and three hidden neurons respectively. 
In addition, the linear term $c_*x$ in the canonical form of a $G$-SNN---if not zero---can be interpreted as two additional hidden neurons. 
It follows that a general $G$-SNN that is a sum of copies of the two irreducible architectures cannot have $3k+1$ hidden neurons for any integer $k$. 
Thus, if we fit a traditional fully-connected SNN with $3k+1$ hidden neurons to a dataset invariant under $60^\circ$-rotations, 
then the fit SNN can be a $G$-SNN if and only if one or more of its hidden neurons are redundant-- 
e.g., one hidden neuron is zeroed out, or four hidden neurons sum to form a linear term, leaving $(3k+1)-4 = 3(k-1)$ hidden neurons corresponding to ``proper'' weight vectors. 
Although this is a rather simple example, it suggests the possibility of more severe or complicated restrictions on numbers of hidden neurons for larger and richer groups $G$. 
In these cases, the redundant hidden neurons could perhaps make it more difficult for the SNN to discover the symmetries in the dataset and thus weaken the model, all at the cost of additional computation. 
We thus conjecture that one factor that determines the optimal number of hidden neurons in traditional SNNs is whether the number admits a $G$-SNN architecture, 
or---going further---how many different $G$-SNN architectures the number admits.

\subsection{Network morphisms}
\label{sec:ai}

Let $f_i^{\PZ}=F^{-1}(\rho^{\PZ}_i)$ for $i=1,2$ be two $G$-SNN architectures. 
If for every $f\in f_2^{\PZ}$ there exists a sequence $\{f_n\in f_1^{\PZ}\}_{n=1}^{\infty}$ that converges to $f$ in the topology of uniform convergence on compact sets,%
\footnote{In this topology, a sequence $\{f_n\}_{n=1}^{\infty}$ of functions is said to converge to $f$ iff it converges uniformly to $f$ on every compact set in the domain.}
then we say $f_2^{\PZ}$ is \emph{asymptotically included} in $f_1^{\PZ}$ and write $f_2^{\PZ}\hookrightarrow f_1^{\PZ}$. 
In the cyclic rotation example in Sec.~\ref{sec:examples:rot}, as already discussed there, the type 2 architecture is asymptoticly included in the type 1 architecture. 
In the cyclic permutation example in Sec.~\ref{sec:examples:perm}, the asymptotic inclusions%
\footnote{These inclusions are ``asymptotic'' because no two rows of the canonical parameter $[W_*\mid b_*]$ can be exactly antiparallel, preventing the degeneration of one architecture into another.}%
 furnish a $3$-partite topology on the space of irreducible architectures (Fig.~\ref{fig:C6_morphisms}); 
here every directed path is an asymptotic inclusion, 
and the individual arcs could be called ``irreducible asymptotic inclusions''. 
This topology provides the necessary structure to perform neural architecture search (NAS), where the irreducible inclusions serve as the \emph{network morphisms}~\citep{wei2016network}. 
In NAS, network morphisms are used to map underfitting architectures to larger ones, after which training resumes; 
the upshot is that the larger architecture need not be re-initialized, thereby significantly cutting computation time. 
In future work, we will run NAS on the space of irreducible $G$-SNN architectures to learn an optimal $G$-SNN in a greedy manner.

\begin{wrapfigure}{R}{0.5\textwidth}
\centering
\includegraphics[trim={0 0 0 1.5cm}, clip, width=0.4\textwidth]{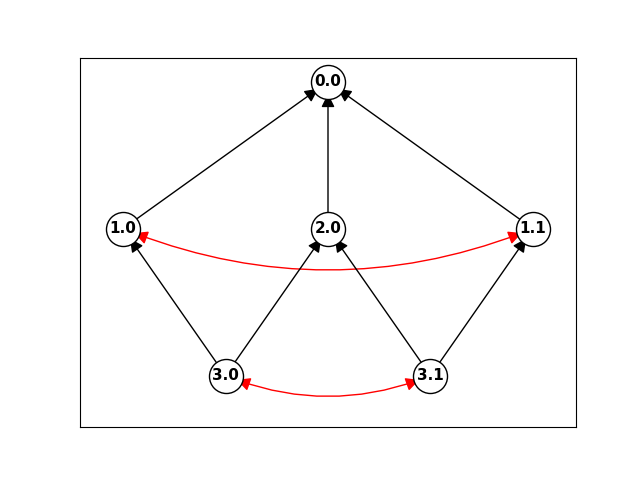}
\caption{\label{fig:C6_morphisms} %
Network morphisms between irreducible $G$-SNN architectures for the cyclic permutation group $G=C_6$. 
Every directed path in black represents an asymptotic inclusion. 
Red doubled-arrowed arcs represent the feasibility of topological tunneling. 
See Supp.~\ref{appendix:examples:perm} for a richer example-- the dihedral permutation group $D_6$.
}
\end{wrapfigure}

Although the above definition of asymptotic inclusion is functional-analytic, 
the following theorem gives a group-theoretic characterization that is more amenable to computation.

\begin{thm} \label{thm:ai}
Let $f^{\PZ}_i = F^{-1}(\rho^{\PZ}_{H_iK_i})$ for $i=1,2$ be two irreducible $G$-SNN architectures. 
Then $f^{\PZ}_2\hookrightarrow f^{\PZ}_1$ iff there exists $(H, K)\in (H_1, K_1)^G$ such that 
$H\leq H_2$, $K\leq K_2$, and $H\cap K_2 = K$.
\end{thm}

The proof (see Supp.~\ref{appendix:ai}) relies on a non-canonical parameterization of $G$-SNNs; a lemma (Lemma~\ref{lemma:ai}) that invokes the Arzel\`{a}-Ascoli Theorem; and Thm.~\ref{thm:main:a}. 
Theorem~\ref{thm:ai} can thus be used to generate network morphisms such as those in Fig.~\ref{fig:C6_morphisms} algorithmicly. 
Observe that as a corollary, since subgroup lattices are connected and subgroup inclusion is transitive, then there are no ``isolated'' $G$-SNNs; 
every $G$-SNN architecture is connected to another by some network morphism.

\subsection{Topological tunneling}
\label{sec:topological}

Recall the discussion in the final paragraph of Sec.~\ref{sec:examples:perm}, where we said architectures i.0 and i.1 coincide in function space only when their weight matrices are zero. 
We see this phenomenon again in the dihedral rotation group example (see Supp.~\ref{appendix:examples:rot}). 
Both of these are instances of the following proposition, which states that architectures corresponding to distinct cohomology classes are in a sense orthogonal.

\begin{prop} \label{prop:orthogonal}
Let $w_1$ and $w_2$ be the first rows of the canonical weight matrices of two irreducible $G$-SNN architectures $F^{-1}(\rho^{\PZ}_{HK_1})$ and $F^{-1}(\rho^{\PZ}_{HK_2})$ where $K_1\neq K_2$. 
Then $w_1^{\top}w_2 = 0$.
\end{prop}

We call the resulting phenomenon ``topological confinement'', and it implies that if, for example, Alice generates a nontrivial dataset using architecture 1.1 for $G=C_6$ (Fig.~\ref{fig:C6_perm}) 
but Bob constructs architecture 1.0 to enforce $G$-invariance (as it is one of the more intuitive architectures), 
then Bob's network will fail to fit to Alice's dataset, even though Bob's network has the right ``size''; 
this suggests that the way we enforce $G$-invariance in a network is important. 
Rather than randomly selecting between architectures 1.0 and 1.1, or resorting to a larger architecture such as 0.0, we propose to allow ``topological tunneling'', 
where the cohomology class of an architecture is transformed by applying the appropriate orthogonal transformation to the top weight vector. 
This allows us to transform one weight-sharing pattern into another in a way analogous to cutting and regluing a M\"{o}bius band to remove the twist. 
Topological tunneling thus introduces ``shortcuts'' between certain points in architecture space, hopefully facilitating NAS (Fig.~\ref{fig:C6_morphisms}). 
We plan to test this in practice in future work.

\section{Conclusion}
\label{sec:conclusion}

We proved Thm.~\ref{thm:main}, which gives a classification of all (irreducible) $G$-SNN architectures with ReLU activation for any finite orthogonal group $G$ acting on the input space. 
The proof is based on a correspondence of every $G$-SNN to a signed perm-rep of $G$ acting on the hidden neurons. 
We also proved Thm.~\ref{thm:ai}, which characterizes the network morphisms between irreducible $G$-SNN architectures and thus---together with Thm.~\ref{thm:main}---completely describes $G$-SNN architecture space. 
A key implication of our theory is the existence of the type 2 $G$-SNN architectures, which to our knowledge have never been explicitly identified in the literature previously.

Various next steps can be taken towards the ultimate goal of $G$-invariant neural architecture design. 
On one hand, we could try to extend Thm.~\ref{thm:main} to deep architectures, which would require us to understand the redundancies of a deep network. 
We could then investigate the behavior and utility of type 2 symmetry constraints in the context of real deep learning benchmark tasks. 
on the other hand, we could first go ahead and investigate NAS on $G$-SNNs. 
For a scalable NAS implementation, we could work out Thm.~\ref{thm:main} for specific families of groups to derive more efficient implementations, 
and we could try to develop an ``algebra'' of $G$-SNNs where $G$ is a (semi)direct product of smaller groups. 
Finally, we could consider what are ``good'' combinations of irreducible $G$-SNN architectures; 
e.g., which sequences of irreducible architectures converge in sum to universal $G$-invariant approximators fastest? 
Perhaps answers to these questions could aid in transforming $G$-invariant neural architecture design from an art to a science.

\begin{ack}
D.A. and J.O. were supported by DOE grant DE-SC0018175. 
D.A. was additionally supported by NSF award No.~2202990.
\end{ack}


\bibliographystyle{plainnat}
\bibliography{references}

\ifpreprint
\else
\section*{Checklist}


\begin{enumerate}

\item For all authors...
\begin{enumerate}
  \item Do the main claims made in the abstract and introduction accurately reflect the paper's contributions and scope?
    \answerYes{}
  \item Did you describe the limitations of your work?
    \answerYes{See Sec.~\ref{sec:intro}.}
  \item Did you discuss any potential negative societal impacts of your work?
    \answerNA{}
  \item Have you read the ethics review guidelines and ensured that your paper conforms to them?
    \answerYes{}
\end{enumerate}

\item If you are including theoretical results...
\begin{enumerate}
  \item Did you state the full set of assumptions of all theoretical results?
    \answerYes{}
        \item Did you include complete proofs of all theoretical results?
    \answerYes{See supplementary material.}
\end{enumerate}

\item If you ran experiments...
\begin{enumerate}
  \item Did you include the code, data, and instructions needed to reproduce the main experimental results (either in the supplemental material or as a URL)?
    \answerYes{}
  \item Did you specify all the training details (e.g., data splits, hyperparameters, how they were chosen)?
    \answerNA{}
        \item Did you report error bars (e.g., with respect to the random seed after running experiments multiple times)?
    \answerNA{}
        \item Did you include the total amount of compute and the type of resources used (e.g., type of GPUs, internal cluster, or cloud provider)?
    \answerNA{}
\end{enumerate}

\item If you are using existing assets (e.g., code, data, models) or curating/releasing new assets...
\begin{enumerate}
  \item If your work uses existing assets, did you cite the creators?
    \answerNA{}
  \item Did you mention the license of the assets?
    \answerNA{}
  \item Did you include any new assets either in the supplemental material or as a URL?
    \answerNA{}
  \item Did you discuss whether and how consent was obtained from people whose data you're using/curating?
    \answerNA{}
  \item Did you discuss whether the data you are using/curating contains personally identifiable information or offensive content?
    \answerNA{}
\end{enumerate}

\item If you used crowdsourcing or conducted research with human subjects...
\begin{enumerate}
  \item Did you include the full text of instructions given to participants and screenshots, if applicable?
    \answerNA{}
  \item Did you describe any potential participant risks, with links to Institutional Review Board (IRB) approvals, if applicable?
    \answerNA{}
  \item Did you include the estimated hourly wage paid to participants and the total amount spent on participant compensation?
    \answerNA{}
\end{enumerate}

\end{enumerate}

\fi

\ifpreprint
\else
\newpage
\setcounter{page}{1}
\fi
\appendix
\part*{Supplementary Material}

\section{Signed permutation representations}
\label{appendix:rho}

\subsection{Classification up to conjugacy}
\label{appendix:rho_clf_conj}

Let $\{e_1,\ldots,e_n\}$ be the standard orthonormal basis set on $\RR^n$. 
For each $i=1,\ldots,n$, define $e_{-i} = -e_i$.

For every $B\in\mathcal{PZ}(n)$, let $\psi_B:\mathcal{PZ}(n)\mapsto\mathcal{PZ}(n)$ be the inner automorphism defined by $\psi_B(A) = B^{-1}AB$. 
Using this notation, two signed perm-reps $\rho,\rho^{\prime}$ are conjugate if there exists $A\in\mathcal{PZ}(n)$ such that $\rho^{\prime} = \psi_A\circ\rho$.

The following proposition states that the property of irreducibility is invariant under conjugation, 
and it thus makes sense to speak of the irreducibility of an entire conjugacy class $\rho^{\PZ}$.

\begin{prop} \label{prop:irrconj}
Let $\rho$ be an irreducible signed perm-rep. 
Then every signed perm-rep $\rho^{\prime}$ conjugate to $\rho$ is also irreducible.
\end{prop}
\begin{proof}
Let $A\in\mathcal{PZ}(n)$ such that $\rho^{\prime}(g) = A^{-1}\rho(g)A\forall g\in G$. 
Note that for each $i=1,\ldots,n$, $Ae_i\in\{e_{\pm 1},\ldots,e_{\pm n}\}$. 
Thus, for every $i,j=1,\ldots,n$, there exists $g\in G$ such that
\begin{align*}
\rho(g)Ae_i &= \pm Ae_j \\
A^{-1}\rho(g)Ae_i &= \pm e_j \\
\rho^{\prime}(g)e_i &= \pm e_j.
\end{align*}
\end{proof}

We next prove a fundamental lemma that establishes a correspondence between irreducible signed perm-reps and the action of $G$ on certain coset spaces. 
This is a generalization of the correspondence between ordinary unsigned permutation representations and the action of $G$ on its coset spaces, which is often formalized in terms of the so-called ``Burnside ring''~\citep{burnside1911theory, bouc2000burnside}. 
This lemma is also the basis for the type 1 vs. type 2 dichotomy of irreducible signed perm-reps mentioned in Sec.~\ref{sec:cohomology}. 

We require two new definitions first. 
An \emph{unsigned permutation representation} (unsigned perm-rep) is a signed perm-rep $\rho$ such that $\rho(g)\in\mathcal{P}(n)\forall g\in G$.
A signed perm-rep $\rho$ is said to be \emph{transitive} on a set $S\subseteq\RR^n$ if for every $v,w\in S$, there exists $g\in G$ such that $\rho(g)v=w$. 

\begin{lemma} \label{lemma:HKU}
Let $\rho$ be an irreducible signed perm-rep. 
Define $K\leq H\leq G$ and $U\in\mathcal{Z}(n)$, $U=\diag(u_1,\ldots,u_n)$, by
\begin{align*}
H &= \{g\in G: \rho(g)e_1 = \pm e_1\} \\
K &= \{g\in G: \rho(g)e_1 = e_1\} \\
u_i &= \begin{cases}
1, & \mbox{ if }\exists g\in G\mid \rho(g)e_1 = e_i \\
-1, & \mbox{ otherwise.}
\end{cases} \\
\end{align*}
Let $\{g_1,\ldots, g_n\}$ be a transversal of $G/H$ such that $U\rho(g_i)U e_1 = e_i$. 
For each $i=1,\ldots,n$, define $g_{-i} = g_i h$ for some $h\in H\setminus K$ if $|H:K|=2$ and $h=1$ if $|H:K|=1$. 
Then:
\begin{enumerate}[label={(\alph*)}]
\item \label{lemma:HKU:a} 
$|H:K|\leq 2$.
\item \label{lemma:HKU:b} 
If $|H:K|=1$, then $U\rho(g)U e_i = e_j$ iff $gg_iK = g_jK$. 
Moreover, $g\rightarrow U\rho(g)U$ is an unsigned perm-rep and is transitive on $\{e_1,\ldots,e_n\}$.
\item \label{lemma:HKU:c} 
If $|H:K|=2$, then $\rho(g)e_i = e_j$ iff $gg_iK = g_jK$. 
Moreover, $U=I_n$ and $\rho$ is transitive on $\{\pm e_1,\ldots, \pm e_n\}$.
\end{enumerate}
\end{lemma}
\begin{proof}
\textbf{(a) } %
If there is no $h\in G$ such that $\rho(h)e_1=-e_1$, then $H=K$, and hence $|H:K|=1$. 
On the other hand, suppose there exists $h\in G$ such that $\rho(h)e_1=-e_1$. 
Then we have
\begin{align*}
H 
&= \{g\in G: \rho(g)e_1 = \pm e_1\} \\
&= \{g\in G: \rho(g)e_1 = e_1\}\cup \{g\in G: \rho(g)e_1 = -e_1\} \\
&= K\cup \{g\in G: \rho(h^{-1})\rho(g)e_1 = -\rho(h^{-1})e_1\} \\
&= K\cup \{g\in G: \rho(h^{-1}g)e_1 = e_1\} \\
&= K\cup hK,
\end{align*}
and hence $|H:K|=2$.

\textbf{(b) } %
Suppose $|H:K|=1$. 
Let $i,j\in\{\pm 1,\ldots, \pm n\}$, 
and suppose there exists $g\in G$ such that $U\rho(g)Ue_i = e_j$. 
We have
\begin{align*}
U\rho(g)U e_i &= e_j \\
U\rho(g)U U\rho(g_i)U e_1 &= U\rho(g_j)U e_1 \\
\rho(g)\rho(g_i)u_1 e_1 &= \rho(g_j)u_1 e_1 \\
\rho(g_j^{-1}gg_i) e_1 &= e_1 \\
g_j^{-1}gg_i &\in K \\
gg_iK &= g_jK.
\end{align*}
This sequence of inferences holds in reverse as well, thus establishing the first part of the claim. 

Since $|H:K|=1$, then $g_{-i}=g_i\forall i\in\{1,\ldots, n\}$, 
and hence the above states that $g\rightarrow U\rho(g)U$ is equivalent to the action of $G$ on $\{g_1K,\ldots, g_nK\}$, 
which is exactly the coset space $G/K$ since $K=H$. 
By the established equivalence, $g\rightarrow U\rho(g)U$ acts transitively on $\{e_1,\ldots, e_n\}$. 
That $g\rightarrow U\rho(g)U$ is an unsigned perm-rep immediately follows from this transitivity.

\textbf{(c) } %
Suppose $|H:K|=2$. 
By the same reasoning as in (b), we can establish that $U\rho(g)U e_i = e_j$ iff $gg_iK = g_jK$. 
Since $|H:K|=2$, then clearly $G/K = \{g_{\pm 1}K,\ldots, g_{\pm n}K\}$. 
We thus have that $g\rightarrow U\rho(g)U$ is equivalent to the action of $G$ on $G/K$. 
By this equivalence, $g\rightarrow U\rho(g)U$ acts transitively on $\{\pm e_1,\ldots, \pm e_n\}$. 
That $U=I_n$ immediately follows from this transitivity, 
and this in turn implies that $\rho(g)e_i = e_j$ iff $gg_iK = g_jK$ and that $\rho$ acts transitively on $\{\pm e_1,\ldots, \pm e_n\}$.
\end{proof}

\begin{remark}
In Lemma~\ref{lemma:HKU}, the irreducibility of the signed perm-rep $\rho$ is necessary to guarantee the existence of $g_i\in G$ such that $U\rho(g_i)U e_1 = e_i$ for each $i=1,\ldots,n$.
\end{remark}

\begin{remark} \label{remark:type}
In Lemma~\ref{lemma:HKU}, the signed perm-rep $\rho$ is said to be of \emph{type 1} (resp. \emph{type 2}) if $|H:K|=1$ (resp. $|H:K|=2$).
\end{remark}

We now prove Thm.~\ref{thm:rho_clf_conj}.

\begin{proof}[Proof of Thm.~\ref{thm:rho_clf_conj}] 
\textbf{(a) } %
Recall by definition of $\mathcal{C}^G_{\leq 2}$, either $|H:K|=1$ or $|H:K|=2$. 
Let $i,j\in\{1,\ldots,n\}$. 
Since $G$ acts transitively on $G/H$, then there exists $g\in G$ such that $gg_iH = g_jH$. 
If $|H:K|=1$, then this is equivalently $gg_iK = g_jK$ so that $\rho_{HK}(g)e_i = e_j$; 
the rep $\rho_{HK}$ is thus irreducible. 
If instead $|H:K|=2$, then we have either $gg_iK = g_jK$ or $gg_iK = g_jhK = g_{-j}K$, 
so that $\rho_{HK}(g)e_i = e_{\pm j} = \pm e_j$; 
the rep $\rho_{HK}$ is still irreducible.

\textbf{(b) } %
Let $\rho^{\prime}$ be an irreducible signed perm-rep. 
We handle type 1 and type 2 as separate cases.

\textbf{(Case 1) } %
Suppose $\rho^{\prime}$ is type 1. 
Then by Lemma~\myref{lemma:HKU}{b}, $\rho^{\prime}$ is conjugate to an unsigned perm-rep. 
We can therefore assume, without loss of generality, that $\rho^{\prime}$ is an unsigned perm-rep and thus corresponds to the action of $G$ on $G/H^{\prime}$ for some $H^{\prime}\leq G$. 
Let $(H, H)^G\in\mathcal{C}^G_{\leq 2}$ be the unique conjugacy class such that $H$ is conjugate to $H^{\prime}$. 
Note that $\rho_{HH}$ is also an unsigned perm-rep and is clearly conjugate to $\rho^{\prime}$, 
thus completing the proof for the type 1 case.

\textbf{(Case 2) } %
Suppose $\rho^{\prime}$ is type 2, and define
\begin{align*}
H^{\prime} &= \{g\in G: \rho^{\prime}(g)e_1 = \pm e_1\} \\
K^{\prime} &= \{g\in G: \rho^{\prime}(g)e_1 = e_1\} \\
g^{\prime}_iK^{\prime} &= \{g\in G: \rho^{\prime}(g)e_1 = e_i\}\forall i=1,\ldots,n.
\end{align*}
Then there exists a unique $(H, K)^G\in\mathcal{C}^G_{\leq 2}$ and $g_*\in G$ such that
\begin{align*}
H^{\prime} &= g_*Hg_*^{-1} \\
K^{\prime} &= g_*Kg_*^{-1}.
\end{align*}
Note that since $\rho^{\prime}$ is type 2, then $|H:K|=|H^{\prime}:K^{\prime}|=2$; 
thus, $G/K = \{g_{\pm 1}K,\ldots, g_{\pm n}K\}$. 
Define $\sigma:G\mapsto \{g_{\pm 1},\ldots,g_{\pm n}\}$ such that $g\in \sigma(g)K$, 
and define a permutation $\pi$ on $\{\pm 1,\ldots,\pm n\}$ such that
\[ g_{\pi(i)} = \sigma(g^{\prime}_i g_*). \]
Let $A\in\mathcal{PZ}(n)$ such that $Ae_i = e_{\pi(i)}$ for each $i$. 
Then we claim $\rho^{\prime} = \psi_A\circ\rho_{HK}$. 
For any $g\in G$ and $i\in\{\pm 1,\ldots,\pm n\}$, 
let $j\in\{\pm 1,\ldots,\pm n\}$ such that $\rho^{\prime}(g)e_i = e_j$. 
Using Lemma~\myref{lemma:HKU}{c}, we have
\begin{align*}
\rho^{\prime}(g)e_i &= e_j \\
gg^{\prime}_iK^{\prime} &= g^{\prime}_jK^{\prime} \\
gg^{\prime}_i g_*Kg_*^{-1} &= g^{\prime}_j g_*Kg_*^{-1} \\
gg^{\prime}_i g_*K &= g^{\prime}_j g_*K \\
g\sigma(g^{\prime}_i g_*)K &= \sigma(g^{\prime}_j g_*)K \\
gg_{\pi(i)}K &= g_{\pi(j)}K \\
\rho_{HK}(g)e_{\pi(i)} &= e_{\pi(j)} \\
\rho_{HK}(g)Ae_i &= Ae_j \\
A^{-1}\rho_{HK}(g)Ae_i &= e_j \\
(\psi_A\circ\rho_{HK})(g)e_i &= e_j.
\end{align*}
This sequence of inferences holds in the reverse direction as well, and hence $\rho^{\prime} = \psi_A\circ\rho_{HK}$ as claimed.
\end{proof}

\begin{remark} \label{remark:transversal}
The transversal $\{g_1,\ldots,g_n\}$ of $G/H$ used in the definition of $\rho_{HK}$ can be recovered from the latter up to $K$. 
Let $\{g^{\prime}_1,\ldots, g^{\prime}_n\}$ be another transversal of $G/H$ such that $\rho_{HK}(g^{\prime}_i)e_1 = e_i$. 
By definition of $\rho_{HK}$, $g^{\prime}_ig_1K = g_iK$. 
Since $g_1\in K$, then $g^{\prime}_iK = g_iK$.
\end{remark}

\subsection{Some useful properties}
\label{appendix:useful}

For every $z\in\{-1, 1\}^n$, define the signed perm-rep
\[ \rho_{HK;z}(g) = \diag(z)\rho_{HK}(g)\diag(z)\forall g\in G. \]
The following proposition and subsequent corollary list some useful properties of the $\rho_{HK;z}$. 
Note that $\rho_{HK} = \rho_{HK;z}$ with $z=\vec{1}$, and hence the statements below hold in particular for the $\rho_{HK}$ as well.

\begin{prop} \label{prop:rho_HKz}
Let $\rho = \rho_{HK;z}$ be an irreducible signed perm-rep. 
Let $Z = \diag(z)$. 
Then the following statements are true:
\begin{enumerate}[label={(\alph*)}] 
\item \label{prop:rho_HKz:a} 
The subgroups $H$ and $K$ satisfy
\begin{align*}
H &= \{g\in G: \rho(g)e_1 = \pm e_1\} \\
K &= \{g\in G: \rho(g)e_1 = e_1\},
\end{align*}
and $\rho$ is of type $|H:K|$.
\item \label{prop:rho_HKz:b} 
If $\rho$ is type 1 ($|H:K|=1$), then $\rho_{HK}(g) = Z\rho(g)Z$ is an unsigned perm-rep that acts transitively on $\{e_1,\ldots, e_n\}$.
\item \label{prop:rho_HKz:c} 
If $\rho$ is type 2 ($|H:K|=2$), then $\rho$ acts transitively on $\{\pm e_1,\ldots, \pm e_n\}$.
\end{enumerate}
\end{prop}
\begin{proof}
\textbf{(a) } %
Define the subgroups
\begin{align*}
H^{\prime} &= \{g\in G: \rho(g)e_1 = \pm e_1\} \\
K^{\prime} &= \{g\in G: \rho(g)e_1 = e_1\}.
\end{align*}
We then have
\begin{align*}
H^{\prime} 
&= \{g\in G: Z\rho_{HK}(g)Ze_1 = \pm e_1\} \\
&= \{g\in G: \rho_{HK}(g)Ze_1 = \pm Ze_1\} \\
&= \{g\in G: \rho_{HK}(g)z_1e_1 = \pm z_1e_1\} \\
&= \{g\in G: \rho_{HK}(g)e_1 = \pm e_1\}.
\end{align*}
By definition of $\rho_{HK}$ in Thm.~\ref{thm:rho_clf_conj}, we have
\begin{align*}
H^{\prime} 
&= \{g\in G: gg_1K = g_{\pm 1}K\} \\
&= \{g\in G: gK = K \mbox{ or } gK = hK\} \\
&= K\cup hK \\
&= H.
\end{align*}
We can similarly show that $K^{\prime}=K$. 
By definition of type in Remark~\ref{remark:type}, $\rho$ is of type $|H^{\prime}:K^{\prime}| = |H:K|$.

\textbf{(b) } %
Suppose $\rho$ is type 1 so that $|H:K|=1$ and hence $H=K$. 
Then $\rho_{HK}(g)e_i = e_j$ iff $gg_iK = g_jK$, 
where $\{g_1,\ldots,g_n\}$ is the transversal of $G/H$ used in the definition of $\rho_{HK}$ in Thm.~\ref{thm:rho_clf_conj}. 
Since $H=K$, however, $\{g_1,\ldots,g_n\}$ is equivalently a transversal of $G/K$, 
and hence we see that the action of $\rho_{HK}$ is equivalent to the action of $G$ on $G/K$. 
As in the proof of Lemma~\myref{lemma:HKU}{b}, this implies the claim.

\textbf{(c) } %
Suppose $\rho$ is type 2. 
Then the claim immediately follows by Lemma~\myref{lemma:HKU}{c}.
\end{proof}

The following corollary results from the combination of Lemma~\ref{lemma:HKU} and Prop.~\ref{prop:rho_HKz}.

\begin{cor} \label{cor:HK}
Let $\rho$ be an irreducible signed perm-rep. 
Define $K\leq H\leq G$ and $z\in\{-1, 1\}^n$ by
\begin{align*}
H &= \{g\in G: \rho(g)e_1 = \pm e_1\} \\
K &= \{g\in G: \rho(g)e_1 = e_1\} \\
z_i &= \begin{cases}
1, & \mbox{ if }\exists g\in G\mid \rho(g)e_1 = e_i \\
-1, & \mbox{ otherwise.}
\end{cases} \\
\end{align*}
Then $\rho = \rho_{HK;z}$. 
Moreover, if $\rho$ is type 2 ($|H:K|=2$), then $z=\vec{1}$ so that $\rho = \rho_{HK}$.
\end{cor}
\begin{proof}
If $\rho$ is type 2, then by Lemma~\myref{lemma:HKU}{c}, $\rho$ is transitive on $\{\pm e_1,\ldots,\pm e_n\}$ so that $z_i=1$ for each $i=1,\ldots,n$. 
Now let $\rho_z(g) = \diag(z)\rho(g)\diag(z)\forall g\in G$. 
Then again by Lemma~\ref{lemma:HKU}, $\rho_z(g)e_i = e_j$ iff $gg_iK = g_jK$; 
however, recalling Thm.~\ref{thm:rho_clf_conj}, this is identical to the definition of $\rho_{HK}$. 
Hence, $\rho_z = \rho_{HK}$, or equivalently $\rho = \rho_{HK;z}$.
\end{proof}

\subsection{Group cohomology}
\label{appendix:cohomology}

For every signed perm-rep $\rho_{HK}$, let $\pi_H:G\mapsto\mathcal{P}(n)$ and $\omega_{HK}:G\mapsto\mathcal{Z}(n)$ be the unique functions satisfying $\rho_{HK}(g) = \omega_{HK}(g)\pi_H(g)\forall g\in G$.%
\footnote{By uniqueness of factorization in a semidirect product, there exist unique functions $\pi:G\mapsto\mathcal{P}(n)$ and $\zeta,\omega:G\mapsto\mathcal{Z}(n)$ such that $\rho(g) = \pi(g)\zeta(g) = \omega(g)\pi(g)\forall g\in G$.}
The following proposition justifies the notation $\pi_H$; i.e., $\pi_H$ does not depend on the choice of $K$ and $z$.

\begin{prop} \label{prop:pi_H}
Let $\rho_{HK}$ be an irreducible signed perm-rep of $G$, 
and let $\pi_{HK}:G\mapsto\mathcal{P}(n)$ and $\zeta_{HK}:G\mapsto\mathcal{Z}(n)$ be the unique functions satisfying $\rho_{HK}(g) = \pi_{HK}(g)\zeta_{HK}(g)\forall g\in G$. 
Then $\pi_{HK}$ is independent of $K$.
\end{prop}
\begin{proof}
As in Remark~\ref{remark:transversal}, let $\{g_1,\ldots,g_n\}$ be a transversal of $G/H$ such that $\rho_{HK}(g_i)e_1 = e_i$ for $i=1,\ldots,n$. 
In general, $g_1\in K$; however, without loss of generality, assume $g_1=1$ so that it is independent of $K$. 
If $g\in G$ and $i,j\in\{\pm 1,\ldots,\pm n\}$ such that $\rho_{HK}(g)e_i = e_j$, 
then define $\pi_{HK}$ and $\zeta_{HK}$ such that
\begin{align*}
\pi_{HK}(g)e_i &= e_{|j|} \\
\zeta_{HK}(g)e_i &= \operatorname{sign}(j).
\end{align*}
It is then easy to verify that $\rho_{HK}(g) = \pi_{HK}(g)\zeta_{HK}(g)\forall g\in G$; 
hence by uniqueness, these are the correct definitions of $\pi_{HK}$ and $\zeta_{HK}$. 
By these definitions, for $g\in G$ and $i,j\in\{1,\ldots,n\}$, $\pi_{HK}(g)e_i = e_j$ iff $gg_iK = g_{\pm j}K$, 
which in turn holds iff $gg_iH = g_jH$. 
This reveals that $\pi_{HK}$ does not depend on $K$ but only $H$.
\end{proof}

The following proposition relates the structure of irreducible signed perm-reps of $G$ to its cohomology.

\begin{prop} \label{prop:cohomology}
For every conjugacy class $\rho_{HK}^{\PZ}$ of irreducible signed perm-reps, 
define the $G$-module $M_H = (\{0, 1\}^n, \pi_H)$ under addition modulo $2$, 
where $n = |G|/|H|$. 
Define $\hat{\omega}_{HK}:G\mapsto M_H$ such that $\hat{\omega}_{HK}(g) = \frac{1}{2}[I-\diag(\omega_{HK}(g))]$. 
Then:
\begin{enumerate}[label={(\alph*)}]
\item \label{prop:cohomology:a} 
The first cohomology group of $G$ with coefficients in $M_H$ is given by%
\footnote{We use the notation $\{\ldots\}_{\neq}$ to emphasize that, during the construction of the set, the enumerated elements are distinct.}
\[ \mathcal{H}^1(G, M_H) = \{[hat{\omega}_{HK}]: K\leq H\mid |H:K|\leq 2\}_{\neq}, \]
where $[\hat{\omega}_{HK}]$ is the set of all cocycles cohomologous to $\omega_{HK}$, 
and where the addition operation satisfies
\[ [\hat{\omega}_{HK}] = [\hat{\omega}_{HK_1}] + [\hat{\omega}_{HK_2}] 
\Leftrightarrow 
K = K_1\cap K_2\cup ((H\setminus K_1)\cap (H\setminus K_2)). \]
\item \label{prop:cohomology:b} 
The partition of the first cohomology group into orbits under the action of the $G$-module automorphism group $\operatorname{aut}(M_H)$ is given by
\[ H^1(G, M_H)/\operatorname{aut}(M_H) = \{\{[\hat{\omega}_{HK^{\prime}}]: (H, K^{\prime})\in (H, K)^G\}: (H, K)^G\in \mathcal{C}^G_{\leq 2}\}. \]
\item \label{prop:cohomology:c} 
$\rho_{HK}$ is type 1 if and only if $\hat{\omega}_{HK}$ is in the zero cohomology class.
\end{enumerate}
\end{prop}
\begin{proof}[Proof of Prop.~\ref{prop:cohomology}] 
\textbf{(a) } %
We first show that every $\hat{\omega}_{HK}$ is a $1$-cocycle by verifying the cocycle condition. 
For $g_1,g_2\in G$, we have
\begin{align*}
\omega_{HK}(g_1g_2)\pi_H(g_1g_2) 
&= \rho_{HK}(g_1g_2) \\
&= \rho_{HK}(g_1)\rho_{HK}(g_2) \\
&= \omega_{HK}(g_1)\pi_H(g_1) \omega_{HK}(g_2)\pi_H(g_2) \\
&= \omega_{HK}(g_1)\pi_H(g_1) \omega_{HK}(g_2) \pi_H(g_1)^{\top} \pi_H(g_1)\pi_H(g_2).
\end{align*}
Equating the factors contained in $\mathcal{Z}(n)$, we have
\[ \omega_{HK}(g_1g_2) = \omega_{HK}(g_1)\pi_H(g_1)\omega_{HK}(g_2)\pi_H(g_1)^{\top}. \]
Writing this in terms of vectors in $M_H$, we obtain the $1$-cocycle condition:
\[ \hat{\omega}_{HK}(g_1g_2) = \hat{\omega}_{HK}(g_1) + \pi_H(g_1)\hat{\omega}_{HK}(g_2), \]
and hence $\hat{\omega}_{HK}$ is a $1$-cocycle.

Next, before proving the main claim, we characterize all cocycles cohomologous to $\hat{\omega}_{HK}$. 
For every $z\in\{-1, 1\}^n$, define the signed perm-rep $\rho_{HK;z}(g) = \diag(z)\rho_{HK}(g)\diag(z)\forall g\in G$, 
and let $\pi_{H;z}:G\mapsto\mathcal{P}(n)$ and $\omega_{HK;z}:G\mapsto\mathcal{Z}(n)$ be the unique functions satisfying $\rho_{HK;z}(g) = \omega_{HK;z}(g)\pi_{H;z}(g)\forall g\in G$. 
We have for all $g\in G$,
\begin{align*}
\rho_{HK;z}(g) &= \diag(z)\rho_{HK}(g)\diag(z) \\
\omega_{HK;z}(g)\pi_{H;z}(g) &= \diag(z)\omega_{HK}(g)\pi_H(g)\diag(z) \\
&= \diag(z)\omega_{HK}(g)\pi_H(g)\diag(z)\pi_H(g)^{\top}\pi_H(g).
\end{align*}
Equating factors in $\mathcal{P}(n)$ and equating factors in $\mathcal{Z}(n)$, we obtain
\begin{align*}
\pi_{H;z}(g) &= \pi_H(g) \\
\omega_{HK;z}(g) &= \diag(z)\omega_{HK}(g)\pi_H(g)\diag(z)\pi_H(g)^{\top}.
\end{align*}
The first of these equations tells us that $\pi_{H;z}$ is independent of $z$, 
and we will thus omit the subscript $z$ in $\pi_{H;z}$ henceforth. 
Writing the second of these equations in terms of vectors in $M_H$, we have
\begin{align*}
\hat{\omega}_{HK;z}(g) &= \hat{z} + \hat{\omega}_{HK}(g) + \pi_H(g)\hat{z} \\
&= (\pi_H(g)\hat{z}-\hat{z}) + \hat{\omega}_{HK}(g),
\end{align*}
where we used the fact that $\hat{z}=-\hat{z} \pmod{2}$. 
Since $g\rightarrow \pi(g)\hat{z}-\hat{z}$ is a coboundary, then $\hat{\omega}_{HK;z}$ is cohomologous to $\hat{\omega}_{HK}$; 
from the above, the converse is also easily verified. 

We thus have
\[ [\hat{\omega}_{HK}] = \{\hat{\omega}_{HK;z}: z\in\{-1, 1\}^n\}, \]
where distinct $z$ do not necessarily imply distinct $\hat{\omega}_{HK;z}$.

We now prove the main claim. 
We first prove that the cohomology classes $[\hat{\omega}_{HK}]$ enumerated over all $K\leq H\mid |H:K|\leq 2$ are distinct. 
Suppose $[\hat{\omega}_{HK_1}] = [\hat{\omega}_{HK_2}]$; 
i.e., $\hat{\omega}_{HK_1}$ and $\hat{\omega}_{HK_2}$ are cohomologous. 
We will show $K_1=K_2$. 
By the above, there exists $z\in\{-1, 1\}^n$ such that $\hat{\omega}_{HK_2} = \hat{\omega}_{HK_1;z}$; 
Converting this back in terms of diagonal matrices and multiplying the resulting equation from the right by $\pi_H$, we obtain $\rho_{HK_2} = \rho_{HK_1;z}$. 
By definition of $\rho_{HK_2}$, we have
\[ \{g\in G: \rho_{HK_2}(g)e_1 = e_1\} = K_2. \]
On the other hand,
\begin{align*}
\{g\in G: \rho_{HK_2}(g)e_1 = e_1\}
&= \{g\in G: \rho_{HK_1;z}(g)e_1 = e_1\} \\
&= \{g\in G: \diag(z)\rho_{HK_1}(g)\diag(z)e_1 = e_1\} \\
&= \{g\in G: \rho_{HK_1}(g)\diag(z)e_1 = \diag(z)e_1\} \\
&= \{g\in G: \rho_{HK_1}(g)z_1e_1 = z_1e_1\} \\
&= \{g\in G: \rho_{HK_1}(g)e_1 = e_1\} \\
&= K_1.
\end{align*}
Ergo, $K_1=K_2$.

We next prove that every $1$-cocycle is contained in one of the cohomology classes $[\hat{\omega}_{HK}]$. 
Let $\hat{\omega}:G\mapsto M_H$ be a $1$-cocycle. 
Then $\rho(g) = \omega(g)\pi_H(g)\forall g\in G$ defines an irreducible signed perm rep. 
It is easy to verify that
\[ H = \{g\in G: \rho(g)e_1 = \pm e_1\}, \]
and define
\[ K = \{g\in G: \rho(g)e_1 = e_1\}. \]
Then by Cor.~\ref{cor:HK}, $\rho = \rho_{HK;z}$ for some $z\in\{-1, 1\}^n$, 
and hence $\hat{\omega} = \hat{\omega}_{HK;z}$ so that $\hat{\omega}\in [\hat{\omega}_{HK}$.

All that is left for (a) is to prove the claimed identity for the addition operation. 
First, however, given a cocycle $\hat{\omega}_{HK;z}$, note that by Prop.~\myref{prop:rho_HKz}{a}, we have
\begin{align*}
K 
&= \{g\in G: \rho_{HK;z}(g)e_1 = e_1\} \\
&= \{g\in G: \omega_{HK;z}(g)\pi_H(g)e_1 = e_1\} \\
&= \{g\in H: \omega_{HK;z}(g)e_1 = e_1\} \\
&= \{g\in H: \omega_{HK;z}(g)_{11} = 1\} \\
&= \{g\in H: \hat{\omega}_{HK;z}(g)_1 = 0\}.
\end{align*}
Now consider the sum of two cohomology classes $[\hat{\omega}_{HK_1}]$ and $[\hat{\omega}_{HK_2}]$. 
Since we have established all elements of the cohomology group, then we know that there exists $K\leq H\mid |H:K|\leq 2$ such that
\[ [\hat{\omega}_{HK}] = [\hat{\omega}_{HK_1}] + [\hat{\omega}_{HK_2}]. \]
Thus, there exists $z\in\{-1, 1\}^n$ such that 
\[ \hat{\omega}_{HK;z} = \hat{\omega}_{HK_1} + \hat{\omega}_{HK_2}. \]
Now by the above, we have
\begin{align*}
K 
&= \{g\in H: \hat{\omega}_{HK;z}(g)_1 = 0\} \\
&= \{g\in H: \hat{\omega}_{HK_1}(g)_1+\hat{\omega}_{HK_2}(g)_1 = 0\} \\
&= \{g\in H: \hat{\omega}_{HK_1}(g)_1 = \hat{\omega}_{HK_2}(g)_1 = 0\} 
\cup \{g\in H: \hat{\omega}_{HK_1}(g)_1 = \hat{\omega}_{HK_2}(g)_1 = 1\} \\
&= K_1\cap K_2 \cup ((H\setminus K_1)\cap (H\setminus K_2)),
\end{align*}
thereby establishing the claim.

\textbf{(b) } %
Let $[\hat{\omega}_{HK_1}]$ and $[\hat{\omega}_{HK_2}]$ be two cohomology classes. 
We must show $(H, K_1)$ is conjugate to $(H, K_2)$ if and only if there exists $P\in\mathcal{P}(n)$ such that $[P, \pi(g)] = 0\forall g\in G$ and $[P\hat{\omega}_{HK_1}] = [\hat{\omega}_{HK_2}]$. 
Suppose $(H, K_1)$ and $(H, K_2)$ are conjugate. 
Then by Thm.~\ref{thm:rho_clf_conj}, $\rho_{HK_1}$ and $\rho_{HK_2}$ are conjugate, 
so that there exist $P\in\mathcal{P}(n)$ and $Z\in\mathcal{Z}(n)$, $Z=\diag(z)$, such that for all $g\in G$,
\begin{align*}
\rho_{HK_2}(g) &= ZP\rho_{HK_1}(g)(ZP)^{-1} \\
\rho_{HK_2}(g) &= ZP\rho_{HK_1}(g)P^{\top}Z \\
\rho_{HK_2;z}(g) &= P\rho_{HK_1}(g)P^{\top} \\
\omega_{HK_2;z}(g)\pi_H(g) &= P\omega_{HK_1}(g)\pi_H(g)P^{\top} \\
\omega_{HK_2;z}(g)\pi_H(g) &= P\omega_{HK_1}(g)P^{\top} P\pi_H(g)P^{\top}.
\end{align*}
Equating the factors in $\mathcal{P}(n)$ and the factors in $\mathcal{Z}(n)$, we obtain
\begin{align*}
\pi_H(g) &= P\pi_H(g)P^{\top} \\
\omega_{HK_2;z}(g) &= P\omega_{HK_1}(g)P^{\top}.
\end{align*}
The first of these equations establishes the commutation $[P, \pi(g)]=0$. 
The second equation implies
\begin{align*}
\hat{\omega}_{HK_2;z}(g) &= P\hat{\omega}_{HK_1}(g) \\
[\hat{\omega}_{HK_2}] &= [P\hat{\omega}_{HK_1}].
\end{align*}
The above steps can be reversed to prove the converse.

\textbf{(c) } %
For every $K\leq H\mid |H:K|\leq 2$, observe that
\[ K\cap H \cup ((H\setminus K)\cap (H\setminus H)) = K. \]
By (a), $[\hat{\omega}_{HK}]$, $H=K$, is thus the zero cohomology class. 
Therefore, $\rho_{HK;z}$ is type 1 ($|H:K|=1$, or $H=K$) if and only if $[\hat{\omega}_{HK}]$ is the zero cohomology class.
\end{proof}

The type 1 vs. type 2 dichotomy is thus rooted in whether a signed perm-rep ``twists'' over $G/H$. 
Proposition~\ref{prop:cohomology} also lets us interpret the notation $\rho_{HK}$: 
The subgroup $H$ determines the coefficient module $M_H$ and hence the cohomology ring, 
and the subgroup $K$ determines the cohomology class in $\mathcal{H}^1(G, M_H)$.

\section{Classification of \texorpdfstring{$G$}{G}-SNNs}
\label{appendix:gsnn}

\subsection{Canonical parameterization}
\label{appendix:canonical}

Let $f:\RR^m\mapsto\RR$ be a continuous piecewise-affine function. 
An \emph{affine region} $X\subseteq\RR^m$ of $f$ is a maximal polytope over which $f$ is affine.

Let $f:\RR^m\mapsto\RR$ be an SNN of the form in Eq.~\ref{eq:snn}, 
and note that $f$ is a continuous piecewise-affine function. 
Then the \emph{signature} $r$ of an affine region $X\subseteq\RR^m$ is the binary vector $r = H(Wx+b)$, 
for any arbitrary choice of $x$ in the interior of $X$ 
and where $H$ is the Heaviside step function (where we set $H(0) = 0$).

The following small proposition establishes the identity given in Eq.~\ref{eq:relulinear}.

\begin{prop} \label{prop:relulinear}
For all $x\in\RR$ and $z\in\{-1, 1\}$,
\[ \relu(x) - \relu(zx) = H(-z)x. \]
\end{prop}
\begin{proof}
It is easy to verify that $\relu(x)-\relu(-x)=x$ for all $x$. 
Now we have two cases:

\textbf{Case 1 ($z=-1$) } %
We have
\begin{align*}
\relu(x) - \relu(zx) 
&= \relu(x) - \relu(-x) \\
&= x \\
&= H[-(-1)]x \\
&= H(-z)x.
\end{align*}

\textbf{Case 2 ($z=1$) } %
We have
\begin{align*}
\relu(x) - \relu(zx) \\
\relu(x) - \relu(x) \\
&= 0 \\
&= H(-1)x \\
&= H(-z)x.
\end{align*}
\end{proof}

We now prove Lemma~\ref{lemma:canonical}.

\begin{proof}[Proof of Lemma~\ref{lemma:canonical}] 
Given access to the data $D = \{(x, f(x)): x\in\RR^m\}$, 
we will show that we can in principle determine $[W_*\mid b_*], a_*, c_*, d_*$ uniquely. 
Since $f$ admits the form in Eq.~\ref{eq:snn}, the set of points at which $f$ is not differentiable is a union of $n_*$ distinct affine spaces each of dimension $m-1$, for a unique $n_*\leq n$. 
From the data $D$, we can in principle determine the equation of each affine space; 
let $w_{*i}^\top x+b_{*i}=0$ be the equation defining the $i$th affine space, 
where $\Vert w_{*i}\Vert=1$. 
Let $W_*\in\RR^{n_*\times m}$ with $i$th row $w_{*i}^\top$ and $b_*\in\RR^{n_*}$ with elements $b_{*i}$. 
Note that no two rows of $[W_*\mid b_*]$ are parallel, as parallel rows would correspond to the same affine space. 
Thus, $[W_*\mid b_*]\in\Theta_{n_*}$. 
Note that the action of any element in $\PZ(n_*)$ on $[W_*\mid b_*]$ leaves the corresponding set of affine spaces invariant; 
we thus assume, without loss of generality, that $[W_*\mid b_*]\in\Omega_{n_*}$, thereby establishing the uniqueness of $[W_*\mid b_*]$. 
The function $f$ now admits the form
\[ f(x) = a_*^\top\relu[Z(W_*x+b_*)] + \tilde{f}(x), \]
for some $a_*\in\RR^{n_*}$, $Z\in\mathcal{Z}(n_*)$, and some differentiable piecewise affine function $\tilde{f}(x):\RR^m\mapsto\RR$. 
Note that $a_{*i}\neq 0$ for each $i=1,\ldots,n_*$; 
otherwise, we could simply delete the $i$th row of $[W_*\mid b_*]$. 
Since $\tilde{f}$ is both piecewise-affine and differentiable, then it is necessarily affine; 
hence there exist $\tilde{c}_*\in\RR^m$ and $\tilde{d}_*\in\RR$ such that $\tilde{f}(x) = \tilde{c}_*^\top x+\tilde{d}_*$ and thus 
\[ f(x) = a_*^\top\relu[Z(W_*x+b_*)] + \tilde{c}_*^\top x+\tilde{d}_*. \]
Now applying Prop.~\ref{prop:relulinear}, we have
\begin{align*}
f(x) 
&= a_*^\top\relu(W_*x+b_*) - a_*^\top H(-Z)(W_*x+b_*) + \tilde{c}_*x+\tilde{d}_* \\
&= a_*^\top\relu(W_*x+b_*) + c_*x+d_*,
\end{align*}
where we define
\begin{align*}
c_* &= \tilde{c}_* - a_*^\top H(-Z)W_* \\
d_* &= \tilde{d}_* - a_*^\top H(-Z)b_*.
\end{align*}

All that remains is to show $a_*$, $c_*$, and $d_*$ are unique. 
We start by showing $a_*$ is unique. 
Consider two adjacent affine regions $X,X^\prime\subset\RR^m$ of $f$, 
where the shared boundary is defined by the $i$th affine space. 
Let $r$ and $r^\prime$ be the signatures of $X$ and $X^\prime$. 
Letting $x$ and $x^\prime$ be two arbitrary points from the interiors of $X$ and $X^\prime$ respectively, 
we have the following difference of gradients with respect to $x$:
\begin{align*}
\nabla f(x^\prime) - \nabla f(x) 
&= W_*^\top\diag(r^\prime)a_* - W_*^\top\diag(r)a_* \\
&= W_*^\top\diag(r^\prime-r)a_*.
\end{align*}
Since $X$ and $X^\prime$ differ only across the $i$th affine space, then all entries of $r^\prime-r$ are zero except the $i$th entry. 
We therefore have
\[ \nabla f(x^\prime) - \nabla f(x) = a_{*i}(r^\prime_i-r_i) w_{*i}. \]
Since $r^\prime-r$ and $w_{*i}$ are nonzero and unique, then we can in principle solve this equation to determine a unique value for $a_{*i}$.

To show $c_*$ is unique, we recall the gradient of $f$ evaluated at the point of differentiability $x\in X$: 
\[ \nabla f(x) = W_*^\top\diag(r)a_* + c_*. \]
Since $a_*$ and $W_*$ have been determined, then we can in principle solve this equation to determine a unique value for $c_*$. 
Once this is done, we can then evaluate $f$ at $x$ and solve for the only remaining unknown $d_*$, thereby determining a unique value for $d_*$ as well.
\end{proof}

\begin{remark}
Suppose $f:\RR^m\mapsto\RR$ admits the form in Eq.~\ref{eq:snn}. 
Then it is possible for there to exist $i$ and $j$ such that $a_i=-a_j$, $w_i=-w_j$, and $b_i=-b_j$. 
In this case, we have
\begin{align*}
a_i\relu(w_i^\top x+b_i) + a_j\relu(w_j^\top x+b_j) 
&= a_i\relu(w_i^\top x+b_i) - a_i\relu[-(w_i^\top x+b_i)] \\
&= a_i(w_i^\top x+b_i),
\end{align*}
which follows from Prop.~\ref{prop:relulinear}. 
Such affine and differentiable terms can thus arise, which is why we include the $c_*x+d_*$ term in Lemma~\ref{lemma:canonical}. 
Moreover, observe that because $[W_*\mid b_*]\in\Theta_{n_*}$ in Lemma~\ref{lemma:canonical}, no two rows of $[W_*\mid b_*]$ are equal or opposites of one another, 
and thus no two hidden neurons can be combined to yield an affine term; 
all affine terms are thus collected in the $c_*x+d_*$ term, which helps to make the canonical form of $f$ unique.
\end{remark}

In general, given a group action on a set, the existence of a fundamental domain is not guaranteed. 
The next proposition guarantees the existence of a fundamental domain in $\Theta_n$ under the action of $\PZ(n)$ by way of a constructive example.

\begin{prop} \label{prop:domain}
Let $\leq_*$ be a total order on $\RR^{m+1}$. 
Let $\Omega$ be the set of all $[W\mid b]\in\Theta_n$ such that the first nonzero entry of each row of $W$ is positive and the rows of $[W\mid b]$ are sorted in ascending order under $\leq_*$. 
Then $\Omega$ is a fundamental domain.
\end{prop}
\begin{proof}
We will show that $\{A\Omega: A\in\PZ(n)\}$ is a partition of $\Theta_n$. 
First, however, let $[W\mid b]\in\Omega$. 
Since the rows of $[W\mid b]$ are nonzero (since the rows of $W$ have unit norm), 
then the action of any non-identity $Z\in\mathcal{Z}(n)$ sends $[W\mid b]$ out of $\Omega$. 
Similarly, since the rows of $[W\mid b]$ are pairwise nonparallel and in particular distinct, then any non-identity $P\in\mathcal{P}(n)$ breaks the ascending order of the rows of $[W\mid b]$ and sends it out of $\Omega$. 
Finally, since no two rows of $[W\mid b]$ are opposites, then the actions of $P$ and $Z$ cannot cancel one another. 
It thus follows that every $A\in\PZ(n)$ sends $[W\mid b]$ out of $\Omega$.

We now proceed to show the elements in the claimed partition are disjoint. 
Let $A,B\in\PZ(n)$, and suppose $A\Omega\cap B\Omega\neq \emptyset$. 
So, let $[W\mid b]\in A\Omega\cap B\Omega$. 
Thus, $A^{-1}[W\mid b]$ and $B^{-1}[W\mid b]$ are both in $\Omega$. 
We also note $(B^{-1}A)A^{-1}[W\mid b] = B^{-1}[W\mid b]$. 
If $B^{-1}A$ is not the identity, then by the above, it sends $A^{-1}[W\mid b]$ out of $\Omega$, so that $B^{-1}[W\mid b]\notin\Omega$. 
Since, however, $B^{-1}[W\mid b]\in\Omega$, then $B^{-1}A = I$ so that $A=B$.

We next show that every $[W\mid b]\in\Theta_n$ belongs to some element of the claimed partition. 
Clearly, there exists $A\in\PZ(n)$ such that $A[W\mid b]\in\Omega$, so that $[W\mid b]\in A^{-1}\Omega$.
\end{proof}

\subsection{\texorpdfstring{$G$}{G}-SNNs and signed perm-reps}
\label{appendix:gsnn2rho}

We prove Lemma~\ref{lemma:gsnn2rho}.

\begin{proof}[Proof of Lemma~\ref{lemma:gsnn2rho}] 
We only prove the forward implication; the converse is then straightforward to verify. 
We write $f$ in its canonical form:
\[ f(x) = a_*^\top\relu(W_*x+b_*) + c_*^\top x+d_*. \]
Let $g\in G$. 
Since $g$ is orthogonal and each row of $W_*$ has unit norm, then so does each row of $W_*g$. 
Moreover, since the transformation $[W_*\mid b_*]\rightarrow [W_*g\mid b_*]$ is invertible and no two rows of $[W_*\mid b_*]$ are parallel, 
then the same is true for the rows of $[W_*g\mid b_*]$. 
Thus, $[W_*g\mid b_*]\in\Theta_{n_*}$, 
and hence there exists a unique matrix $\rho(g)\in\PZ(n_*)$ such that $[W_*g\mid b_*]\in \rho(g)\Omega_{n_*}$. 
Let $\pi(g)\in\mathcal{P}(n_*)$ and $\zeta(g)\in\mathcal{Z}(n_*)$ such that $\rho(g) = \pi(g)\zeta(g)$. 
We have
\begin{align*}
f(gx) 
&= a_*^\top\relu(W_*gx+b_*) + c_*^\top gx+d_* \\
&= a_*^\top\relu[\rho(g)(\rho(g)^{-1}W_*gx+\rho(g)^{-1}b_*)] + c_*^\top gx+d_* \\
&= a_*^\top\relu[\pi(g)\zeta(g)(\rho(g)^{-1}W_*gx+\rho(g)^{-1}b_*)] + c_*^\top gx+d_* \\
&= a_*^\top\pi(g)\relu[\zeta(g)(\rho(g)^{-1}W_*gx+\rho(g)^{-1}b_*)] + c_*^\top gx+d_*.
\end{align*}
Using Prop.~\ref{prop:relulinear}, this is
\begin{align*}
f(gx) 
&= a_*^\top\pi(g)\relu(\rho(g)^{-1}W_*gx+\rho(g)^{-1}b_*) - a_*^\top H(-\zeta(g))(\rho(g)^{-1}W_*x+\rho(g)^{-1}b_*) + c_*^\top gx+d_* \\
&= a_*^\top\pi(g)\relu(\rho(g)^{-1}W_*gx+\rho(g)^{-1}b_*) + [c_*^\top g - a_*^\top H(-\zeta(g))\rho(g)^{-1}W_*]x + [d_* - a_*^\top H(-\zeta(g))\rho(g)^{-1}b_*].
\end{align*}
Note that $\rho(g)^{-1} [W_*\mid b_*]\in\Omega_{n_*}$. 
Since $f$ is $G$-invariant, then $f(gx)=f(x)\forall x\in\RR^m$. 
By uniqueness of canonical parameters with respect to the fundamental domain $\Omega_{n_*}$ (Lemma~\ref{lemma:canonical}), 
the canonical parameters of the SNNs $f$ and $f\circ g$ must be equal. 
We thus obtain the constraints
\begin{align*}
W_* &= \rho(g)^{-1}W_* g \\
a_*^\top &= a_*^\top \pi(g) \\
b_* &= \rho(g)^{-1}b_* \\
c_*^\top &= c_*^\top g - a_*^\top H(-\zeta(g))\rho(g)^{-1}W_*g \\
d_* &= d_* - a_*^\top H(-\zeta(g))\rho(g)^{-1}b_*.
\end{align*}
The first three constraints are clearly equivalent to the ones on $W_*$, $a_*$, and $b_*$ claimed in the lemma statement; 
we thus take these as established. 
By the established $W_*$ and $b_*$ constraints, the $c_*$ and $d_*$ constraints simplify to
\begin{align*}
c_*^\top &= c_*^\top g - a_*^\top H(-\zeta(g))W_* \\
d_* &= d_* - a_*^\top H(-\zeta(g))b_*.
\end{align*}
Now since $\zeta(g)$ is a diagonal matrix with $\pm 1$ along its diagonal, then we have
\[ H(-\zeta(g)) = \frac{1}{2}(I-\zeta(g)). \]
Using the established $a_*$ constraint, we have
\begin{align*}
a_*^\top H(-\zeta(g))
&= \frac{1}{2}a_*^\top (I-\zeta(g)) \\
&= \frac{1}{2} a_*^\top (I - \pi(g)\zeta(g)) \\
&= \frac{1}{2}a_*^\top (I - \rho(g)).
\end{align*}
By the established $b_*$ constraint, we have $(I-\rho(g))b_* = b_*-b_* = 0$; 
we thus see that the above constraint on $d_*$ is trivially satisfied. 
By the established $W_*$ constraint, the constraint on $c_*$ becomes
\begin{align*}
c_*^\top 
&= c_*^\top g - \frac{1}{2}a_*^\top (I-\rho(g))W_* \\
&= c_*^\top g - \frac{1}{2}a_*^\top W_* (I-g) \\
g^\top c_* &= c_* + \frac{1}{2}(I-g^\top)W_*^\top a_*.
\end{align*}
Since this holds for all $g\in G$, then we may substitute $g^\top$ with $g$ to establish the claimed constraint on $c_*$.

Finally, we prove that $\rho:G\mapsto\PZ(n_*)$ is a homomorphism. 
Let $g_1,g_2\in G$. 
By the established constraint on $W_*$, we have
\begin{align*}
\rho(g_1)\rho(g_2) [W_*\mid b_*] 
&= \rho(g_1) [W_*g_2\mid b_*] \\
&= [W_*g_1g_2\mid b_*].
\end{align*}
On the other hand, by definition of $\rho$, we have $\rho(g_1g_2) [W_*\mid b_*] = [W_*g_1g_2\mid b_*]$. 
Thus, $[W_*g_1g_2\mid b_*]$ is thus an element of both $\rho(g_1g_2)^{-1}\Omega_{n*}$ and of $[\rho(g_1)\rho(g_2)]^{-1}\Omega_{n*}$, 
which in turn implies $\rho(g_1g_2) = \rho(g_1)\rho(g_2)$.
\end{proof}

The next proposition states that every $G$-SNN can be written as a sum of irreducible $G$-SNNs, 
thereby simplifying the classification problem of $G$-SNNs to that only of irreducible $G$-SNNs. 
Recall the notation introduced in Sec.~\ref{sec:gsnn2rho}.

\begin{prop} \label{prop:f_decomposition}
Every $G$-SNN admits a decomposition into a sum of irreducible $G$-SNNs.
\end{prop}
\begin{proof}
Let $f\in\SNN(G)$, and let $\rho\in F(f)$. 
Let $n_*$ be the degree of $\rho$; 
i.e., the number of rows of the canonical weight matrix of $f$. 
Then partition $\{e_1,\ldots,e_{n*}\}$ into orbits such that $e_i$ and $e_j$ belong to the same orbit if and only if there exists $g\in G$ such that $\rho(g)e_i = \pm e_j$. 
Without loss of generality, select $\rho\in F(f)$ such that each orbit consists of consecutive elements 
(this is done by an appropriate conjugation of $\rho$); 
i.e., each orbit has the form $\{e_i,e_{i+1},\ldots,e_{i+j}\}$. 
Now write $f$ in canonical form such that the corresponding signed perm-rep by Lemma~\ref{lemma:gsnn2rho} is $\rho$:
\[ f(x) = a_*^\top\relu(W_*x+b_*) + c_*^\top x+d_*. \]
For each $i=1,\ldots,k$, where we have $k$ orbits, define the $G$-SNN $f_i$ by taking only the elements $a_{*j}$ of $a_*$, rows $w_{*j}^\top$ of $W_*$, and elements $b_{*j}$ of $b_*$ such that $e_j$ belongs to the $i$th orbit; 
include the affine term $c_*^\top x+d_*$ only in $f_k$. 
Then clearly $f = f_1+\ldots f_k$, where each $f_i$ is $G$-invariant and irreducible.
\end{proof}

\subsection{The classification theorem}
\label{appendix:thm}

This section gives a proof for Thm.~\ref{thm:main}. 
Recall the following notation introduced in Sec.~\ref{sec:thm}: 
If $A$ is a linear operator (resp.~set of linear operators), then let $P_A$ be the orthogonal projection operator onto the vector subspace that is pointwise-invariant under the action of $A$ (resp.~all elements of $A$). 
Note that if $A$ is a finite orthogonal group, then~\citep[sec. 2.6]{serre1977linear}
\[ P_A = \frac{1}{|A|}\sum_{a\in A} a. \]
In addition, if $P_1,P_2$ are two orthogonal projection operators, then let $P_1\cap P_2$ be the orthogonal projection operator onto $\ran(P_1)\cap\ran(P_2)$.

Before proving Thm.~\ref{thm:main}, we need to state and prove two lemmas. 
The first of these appears next.

\begin{lemma} \label{lemma:2I+h}
Let $K\leq H<\mathcal{O}(m)$ be two finite orthogonal groups such that $|H:K|=2$. 
Let $h\in H\setminus K$. 
Then $P_K\cap P_{2I+h} = P_K-P_H$.
\end{lemma}
\begin{proof}
Since $|H:K|=2$, then $K\trianglelefteq H$, 
and hence $H/K=\{K,hK\}$ is a bona fide group. 
Thus, there exists an isomorphism $\pi:H/K\mapsto\ZZ_2$, 
where we define $\ZZ_2=\{-1, 1\}$ under multiplication. 
Since we require $\ker(\pi)=K$, then we have
\[ \pi(h) = 
\begin{cases}
1, & \mbox{ if } h\in K \\
-1, & \mbox{ otherwise.}
\end{cases} \]
Now let
\[ V = \ran(P_K\cap P_{2I+h}) = \{v\in\RR^m: Kv=v, hKv=-v\}. \]
We thus see that $(\pi, V)$ is a representation of $H$, 
and since $\pi$ is scalar-valued, then $(\pi, V)$ is a direct sum of copies of a single complex-irreducible representation (irrep) of $H$. 
Noting that $\pi$ is its own complex-irreducible character, we have the orthogonal projection
\begin{align*}
P_K\cap P_{2I+h} 
&= \frac{1}{|H|}\sum_{h\in H}\pi(h)h \\
&= \frac{1}{|H|}\sum_{g\in K}\pi(g)g + \frac{1}{|H|}\sum_{g\in hK}\pi(g)g \\
&= \frac{1}{|H|}\sum_{g\in K}g - \frac{1}{|H|}\sum_{g\in hK}g \\
&= \frac{1}{|H|}\sum_{g\in K}g - \frac{1}{|H|}\left(\sum_{g\in H}g - \sum_{g\in K}g\right) \\
&= \frac{2}{|H|}\sum_{g\in K}g - \frac{1}{|H|}\sum_{g\in H}g \\
&= \frac{1}{|K|}\sum_{g\in K}g - \frac{1}{|H|}\sum_{g\in H}g \\
&= P_K - P_H.
\end{align*}
\end{proof}

The second lemma, appearing below, will be used to characterize the condition $[W_*\mid b_*]\in\Theta_{n_*}$ appearing in Lemma~\ref{lemma:canonical}.

\begin{lemma} \label{lemma:distinct}
Let $J\leq G$ and $\{g_1,\ldots,g_n\}$ a transversal of $G/J$ with $g_1\in J$. 
Let $V\leq\ran(P_J)$ be a vector subspace, 
and let $P_V$ be the orthogonal projection operator onto $V$. 
Then there exists $w\in V$ such that $g_1w,\ldots,g_nw$ are distinct vectors if and only if $\st_G(P_V) = J$.
\end{lemma}
\begin{proof}
First we note that if $w\in V$, then $g_1w,\ldots,g_nw$ are distinct iff $g_i w\neq w\forall i\in\{2,\ldots,n\}$; 
to see this, we prove the equivalent statement that $g_1w,\ldots,g_nw$ are not distinct iff $g_iw=w$ for some $i\in\{2,\ldots,n\}$. 
For the reverse implication, $g_iw=w$ is equivalently $g_i w=g_1w$, since $g_1\in J$ and $w\in\ran(P_J)$; 
$g_1w$ and $g_iw$ are thus not distinct. 
For the forward implication, suppose $g_iw=g_j w$ for some distinct $i,j\in\{1,\ldots,n\}$. 
Then there exists $k\in\{2,\ldots,n\}$ and $g\in J$ such that $g_kg = g_j^{-1}g_i$. 
We then have
\begin{align*}
g_iw &= g_jw \\
g_j^{-1}g_iw &= w \\
g_kgw &= w \\
g_kw &= w.
\end{align*}

We now prove the stated lemma. 
Define the vector subspaces
\[ V_i = \{v\in V: g_iv = v\}\forall i\in\{2,\ldots,n\}. \]
We have
\begin{align*}
\exists w\in V\mid g_iw\neq w\forall i\in\{2,\ldots,n\} 
&\Leftrightarrow 
\exists w\in V\mid w\notin V_i\forall i\in\{2,\ldots,n\} \\
&\Leftrightarrow 
\exists w\in V\setminus\bigcup_{i=2}^n V_i \\
&\Leftrightarrow 
V_i<V\forall i\in\{2,\ldots,n\} \\
&\Leftrightarrow 
g_i\notin\st(P_V)\forall i\in\{2,\ldots,n\} \\
&\Leftrightarrow 
\st(P_V)=J.
\end{align*}
\end{proof}

We now prove Thm.~\ref{thm:main}.

\begin{proof}[Proof of Thm.~\ref{thm:main}] 
\textbf{(b) } %
Suppose $\rho_{HK}^{\PZ}\in\ran(F)$. 
We first prove the forward implication. 
Suppose $f\in F^{-1}(\rho_{HK}^{\PZ})$. 
Then the canonical parameters of $f$ satisfy Eqs.~\ref{eq:gsnn2rho:W}-\ref{eq:gsnn2rho:c}, 
where the signed perm-rep $\rho$ in Lemma~\ref{lemma:gsnn2rho} satisfies $\rho\in \rho_{HK}^{\PZ}$. 
By an appropriate choice of fundamental domain $\Omega_{n_*}$, we can assume without loss of generality that $\rho = \rho_{HK}$. 
We proceed to prove the claimed expressions for the canonical parameters of $f$.

\textbf{Expression for $a_*$: } %
Regardless of its type, $\rho$ is transitive on $\{e_1,\ldots,e_{n_*}\}$, and thus so is $\pi$. 
Hence, by Eq.~\ref{eq:gsnn2rho:a}, $a_*$ is a constant vector. 
That $a\neq 0$ follows from the definition of the canonical parameter $a_*$ in Lemma~\ref{lemma:canonical}.

\textbf{Expression for $b_*$: } %
If $\rho$ is type 1, then it is an irreducible unsigned perm-rep and is transitive on $\{e_1,\ldots,e_{n_*}\}$. 
Thus, by Eq.~\ref{eq:gsnn2rho:b}, $b_*$ is a constant vector. 
On the other hand, if $\rho$ is type 2, then it is transitive on $\{\pm e_1,\ldots,\pm e_{n_*}\}$. 
In particular, for every $i=1,\ldots,n_*$, there exists $g\in G$ such that $\rho(g)e_i = -e_i$. 
Hence, $b_{*i} = -b_{*i}$ for every $i$, so that $b_*=0$.

\textbf{Expression for $W_*$: } %
By Prop.~\myref{prop:rho_HKz}{a}, the subgroup $K$ satisfies
\[ K = \{g\in G: \rho(g)e_1 = e_1\}. \]
Thus, by Eq~\ref{eq:gsnn2rho:W}, the first row $w^{\top}$ of $W_*$ satisfies $w^{\top} = w^{\top}g\forall g\in K$, 
or equivalently $gw = w\forall g\in K$. 
Thus, $w\in\ran(P_K)$.

In addition, if $\rho$ is type 2, then by Prop.~\ref{prop:rho_HKz}{a}, we have $hw = -w\forall h\in H\setminus K$. 
Given any choice of $h\in H\setminus K$, we have $hK = H\setminus K$. 
We thus have
\begin{align*}
hKw &= -w \\
hw &= -w \\
(2I+h)w &= w.
\end{align*}
Combining this with $w\in\ran(P_K)$, we have $w\in\ran(P_K\cap P_{2I+h})$. 
By Lemma~\ref{lemma:2I+h}, we obtain $w\in\ran(P_K-P_H)$. 
Combining the results for both types 1 and 2, we establish $w\in\ran(P_K-\tau P_H)$. 
That $\Vert w\Vert=1$ follows from the definition of the canonical parameter $W_*$ in Lemma~\ref{lemma:canonical}.

Now let $w_1^{\top},\ldots,w_{n_*}^{\top}$ be the rows of $W_*$. 
Since $G$ and $\rho(G)$ are both orthogonal, then $\rho(g_i^\top)e_i = e_1$, 
and hence the first row of $\rho(g_i^\top)W_*$ is $w_i^\top$. 
By Eq.~\ref{eq:gsnn2rho:W}, the first row of $W_*g_i^{\top}$ is $w_i^\top$ as well; 
thus, since $w_1=w$, then $w^\top g_i^\top = w_i^\top$, 
or equivalently $w_i = g_i w$.

\textbf{Expression for $c_*$: } %
We rewrite Eq.~\ref{eq:gsnn2rho:c} as
\[ (I-g)c_* = -\frac{a}{2}(I-g)W_*\vec{1}\forall g\in G, \]
where we have used Eq.~\ref{eq:main:a}. 
This is equivalently expressed as
\[ (I-P_G)c_* = -\frac{a}{2}(I-P_G)W_*\vec{1}. \]
We focus on the term
\[ (I-P_G)W_*\vec{1} = (I-P_G)(g_1+\ldots +g_n)w. \]
Since $P_G$ is an average over all $g\in G$, then $P_G g_i=P_G\forall i\in\{1,\ldots,n\}$. 
We thus have
\[ (I-P_G)W_*\vec{1} = (g_1+\ldots +g_n)w - n P_G w. \]
If $\rho$ is type 1 so that $w\in\ran(P_K)=\ran(P_H)$, then
\begin{align*}
(g_1+\ldots +g_n)w 
&= (g_1+\ldots +g_n)P_H w \\
&= n P_G w,
\end{align*}
so that $(I-P_G)W_*\vec{1} = 0$. 
On the other hand, if $\rho$ is type 2, then $hw=-w\forall h\in H\setminus K$; 
thus, $w$ cannot be fixed under all of $G$, so that $P_G w=0$ and hence
\[ (I-P_G)W_*\vec{1} = (g_1+\ldots +g_n)w = W_*\vec{1}. \]
Combining the results for both types 1 and 2, we have $(I-P_G)W_*\vec{1} = \tau W_1\vec{1}$ and thus
\[ (I-P_G)c_* = -\frac{1}{2}a\tau W_*\vec{1}. \]
Since we already know the right-hand side is in $\ran(I-P_G)$, then we obtain the expression for $c_*$ as claimed.

For the reverse implication, let $f$ be a $G$-SNN whose canonical parameters satisfy Eqs.~\ref{eq:main:W}-\ref{eq:main:c}. 
Then it is easy to see that the canonical parameters of $f$ also satisfy Eqs.~\ref{eq:gsnn2rho:W}-\ref{eq:gsnn2rho:c}. 
Lemma~\ref{lemma:gsnn2rho} thus implies $F(f) = \rho_{HK}^{\PZ}$.

\textbf{(a) } %
By part (b) of this theorem, $\rho_{HK}^{\PZ}\in\ran(F)$ iff there exists a $G$-SNN $f$ whose canonical parameters satisfy Eqs.~\ref{eq:main:W}-\ref{eq:main:c}. 
Without loss of generality, we assume $a=1$ in Eq.~\ref{eq:main:a} and $c=0$ in Eq.~\ref{eq:main:c}. 
Then $\rho_{HK}^{\PZ}\in\ran(F)$ iff there exists $[W_*\mid b_*]\in\Theta_{n_*}$ such that $W_*$ satisfies Eq.~\ref{eq:main:W} and $b_*$ satisfies Eq.~\ref{eq:main:b}. 
We now have two separate cases depending on the type of $\rho_{HK}$.

\textbf{Case 1: } %
Suppose $\rho_{HK}$ is type 1. 
Without loss of generality, we assume $b\neq 0$ in Eq.~\ref{eq:main:b}. 
Then $[W_*\mid b_*]$ has pairwise nonparallel rows iff $W_*$ has distinct rows. 
By Eq.~\ref{eq:main:W}, $\rho_{HK}^{\PZ}\in\ran(F)$ iff there exists $w\in\ran(P_K)$, such that $g_1w,\ldots,g_nw$ are distinct; 
note this $w$ is necessarily nonzero, and hence without loss of generality, we assume $\Vert w\Vert =1$. 
Noting that $H=K$ since $\rho_{HK}$ is type 1, 
and invoking Lemma~\ref{lemma:distinct} with $J=K$ and $V=\ran(P_K)$, we establish the claim.

\textbf{Case 2: } %
Suppose $\rho_{HK}$ is type 2. 
Then $b=0$ in Eq.~\ref{eq:main:b}, 
and $[W_*\mid b_*]$ has pairwise nonparallel rows iff so does $W_*$. 
Thus, by Eq.~\ref{eq:main:W}, $\rho_{HK}^{\PZ}\in\ran(F)$ iff there exists $w\in\ran(P_K-P_H)$ such that $g_1w,\ldots,g_nw$ are pairwise nonparallel; 
note this $w$ is necessarily nonzero, and hence without loss of generality we can assume $\Vert w\Vert=1$. 
For any $h\in H\setminus K$, we have $-g_iw = g_ihw = g_{-i}w\forall i\in\{1,\ldots,n\}$. 
Moreover, $\{g_{\pm 1},\ldots,g_{\pm n}\}$ is a transversal of $G/K$. 
Invoking Lemma~\ref{lemma:distinct} with $J=K$ and $V=\ran(P_K-P_H)$, we thus establish the claim.
\end{proof}

\section{Examples}
\label{appendix:examples}

\subsection{Irreducible architecture count}
\label{appendix:examples:count}

\begin{table}
\centering
\caption{\label{table:numbers} %
Ratio of the number of irreducible $G$-SNN architectures to the number of irreducible signed perm-reps of each type (1 vs. 2) for every group $G$, $|G|\leq 8$, up to isomorphism. 
The particular representations used for each group are described in the main text.
}
\begin{tabular}{ccc|ccc}
\toprule
$G$ & Type 1 & Type 2 & $G$ & Type 1 & Type 2 \\
\midrule
$C_2$ & 2/2 & 1/1 & $\{e\}$ & 1/1 & 0/0 \\
$C_3$ & 2/2 & 0/0 & $C_2^2$ & 4/5 & 3/6 \\
$C_4$ & 3/3 & 2/2 & $C_2^3$ & 8/16 & 7/35 \\
$C_5$ & 2/2 & 0/0 & $C_2\times C_4$ & 6/8 & 5/11 \\
$C_6$ & 4/4 & 2/2 & $D_3$ & 3/4 & 1/2 \\
$C_7$ & 2/2 & 0/0 & $D_4$ & 5/8 & 7/13 \\
$C_8$ & 4/4 & 3/3 & $Q_8$ & 6/6 & 7/7 \\
\bottomrule
\end{tabular}
\end{table}

Using our code implementation, we enumerated all irreducible $G$-SNN architectures for every group $G$, $|G|\leq 8$, up to isomorphism. 
For each group, we consider only one particular permutation representation defined as follows: 
First, let $[i_1,\ldots,i_n]$ denote the permutation on the orthonormal basis $\{e_1,\ldots,e_n\}$ where $e_j\rightarrow e_{i_j}$. 
We then represent the cyclic group $C_n$ by the set of cyclic permutations generated by $[2,\ldots,n,1]$, 
and we represent the dihedral group $D_n$ as the group generated by $C_n$ together with the reversing permutation $[n,n-1,\ldots,1]$. 
We represent the direct product of groups by the direct sum of the factor groups; 
e.g., if $G_1$ acts on $[1,\ldots,n_1]$ and $G_2$ acts on $[1,\ldots,n_2]$, then $G_1\times G_2$ acts on $[1,\ldots,n_1+n_2]$ with $G_1$ acting on the first $n_1$ elements and $G_2$ acting on the last $n_2$ elements. 
Finally, we represent the quaternian group $Q_8$ in terms of the following generators:
\begin{align*}
i &= [3, 4, 2, 1, 7, 8, 6, 5] \\
j &= [5, 6, 8, 7, 2, 1, 3, 4] \\
k &= [7, 8, 5, 6, 4, 3, 2, 1].
\end{align*}

For each group $G$, we report the ratio of the number of irreducible $G$-SNN architectures of each type to the number of irreducible signed perm reps of the respective type (Table~\ref{table:numbers}). 
We see that there are generally fewer type 2 architectures---which are the topologically nontrivial ones---than type 1 architectures, although the number of type 2 architectures is not negligible. 
We also observe that---especially for the direct products of groups---there is a large number of irreducible signed perm reps that do not satisfy the condition in Thm.~\myref{thm:main}{a}; 
this is likely because in the rejected architectures, some of the weight vectors are constrained such that the architecture is equivalent to a smaller architecture already enumerated. 
This trend also motivates the need for more intuition about the condition in Thm.~\myref{thm:main}{a}.

\subsection{The dihedral permutation group}
\label{appendix:examples:perm}

\begin{table}
\centering
\caption{\label{table:D6} %
Subgroups $K_j\leq H_i\leq G$ of the dihedral permutation group $G=D_6$ such that the irreducible signed perm rep $\rho_{H_i K_j}$ admits a corresponding irreducible $G$-SNN architecture named ``i.j'' in Figs.~\ref{fig:D6_perm:1}-\ref{fig:D6_perm:3}. 
In each row, $|H_i:K_0|=1$ and $|H_i:K_j|=2$ for $j\geq 1$. 
The generators $r$ and $t$ are defined in Eqs.~\ref{eq:r}-\ref{eq:t}.
}
\begin{tabular}{ccccc}
\toprule
\quad & $K_0$ & $K_1$ & $K_2$ & $K_3$ \\
\midrule
$H_0$ & $\brak{e}$ & & & \\
$H_1$ & $\brak{r^3}$ & $\brak{e}$ & & \\
$H_2$ & $\brak{t}$ & $\brak{e}$ & & \\
$H_3$ & $\brak{r^3t}$ & $\brak{e}$ & & \\
$H_4$ & $\brak{r^3, t}$ & $\brak{r^3}$ & $\brak{t}$ & $\brak{r^3t}$ \\
$H_5$ & $\brak{r^2, rt}$ & & & \\
$H_6$ & $D_6$ & $\brak{r^2, t}$ & & \\
\bottomrule
\end{tabular}
\end{table}

\begin{figure}
\centering
\begin{minipage}{\textwidth}
\begin{minipage}{0.5\textwidth}
\centering
Architecture 0.0

\includegraphics[width=0.5\textwidth]{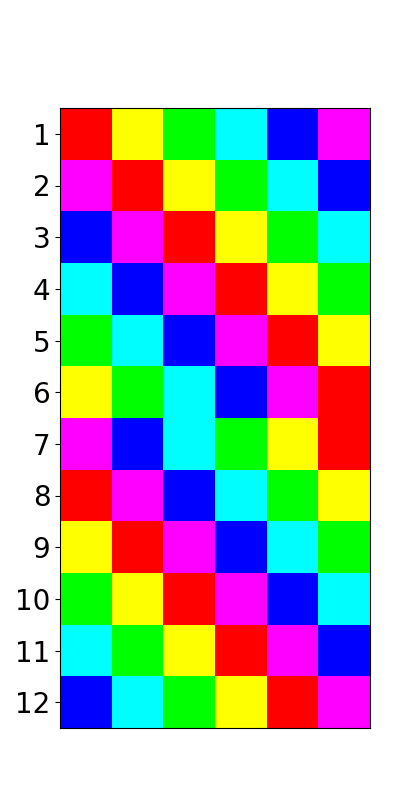}%
\includegraphics[width=0.5\textwidth]{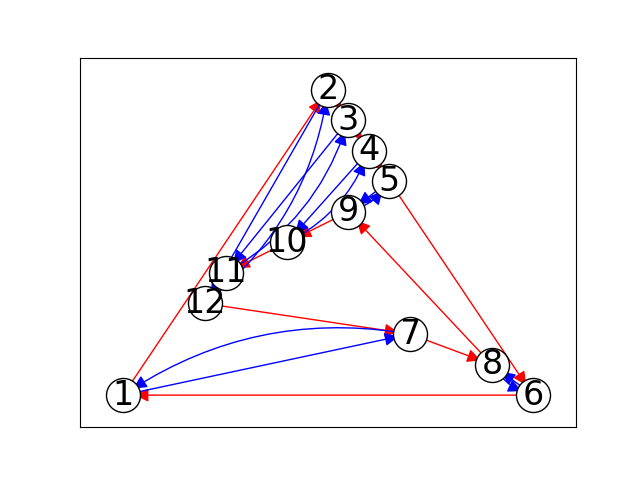}
\end{minipage}%
\begin{minipage}{0.5\textwidth}
\centering
Architecture 5.0

\includegraphics[width=0.5\textwidth]{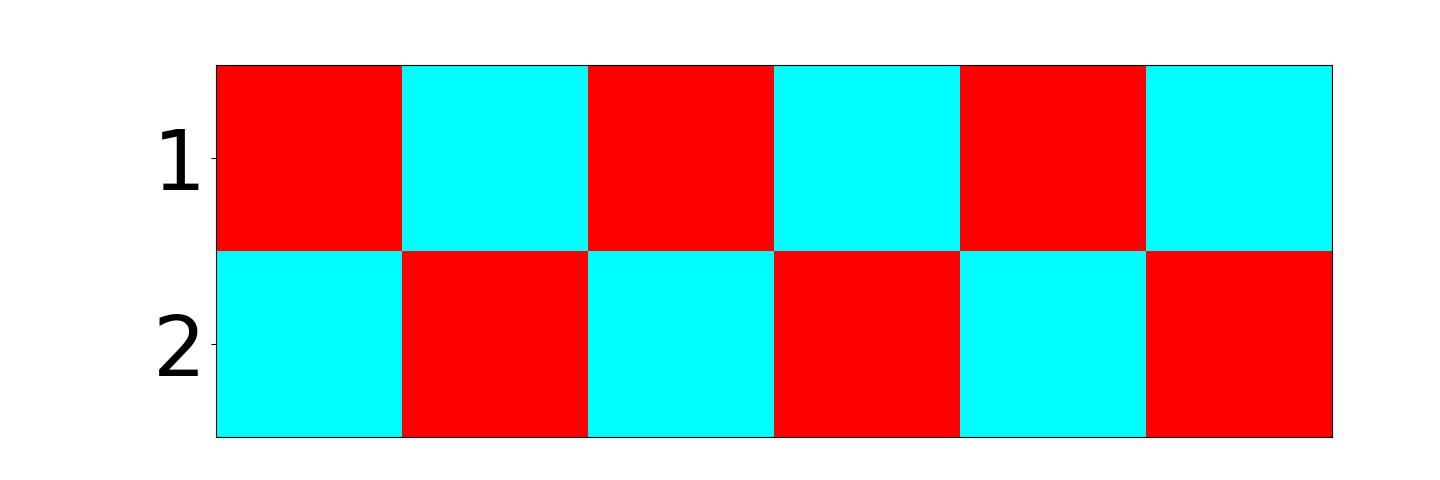}%
\includegraphics[width=0.5\textwidth]{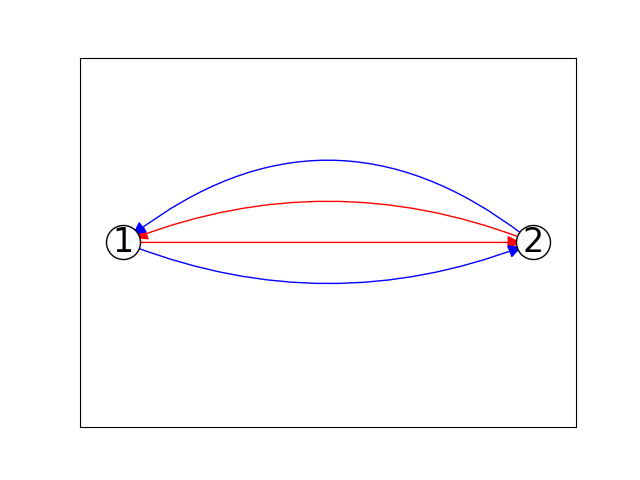}
\end{minipage}
\end{minipage}
\caption{\label{fig:D6_perm:1} %
Constraint patterns of the weight matrices and illustrations of the cohomology classes of two irreducible $G$-SNN architectures for the dihedral permutation group $G=D_6$. 
These are the only two architectures with no partnering type 2 architectures. 
Interpretation is the same as in Fig.~\ref{fig:C6_perm}. 
Red (resp. blue) arcs represent the action of the generator $r$ (resp. $t$) of $D_6$ (see Eqs.~\ref{eq:r}-\ref{eq:t}). 
See Table~\ref{table:D6} to interpret the names ``architecture i.j''.
}
\end{figure}

\begin{figure}
\centering
\begin{minipage}{\textwidth}
\begin{minipage}{0.5\textwidth}
\centering
Architecture 1.0

\includegraphics[width=0.5\textwidth]{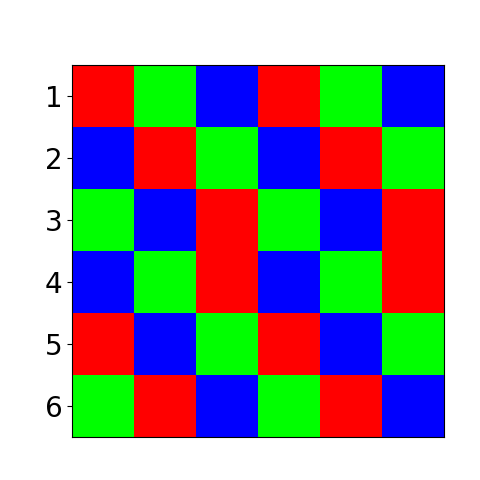}%
\includegraphics[width=0.5\textwidth]{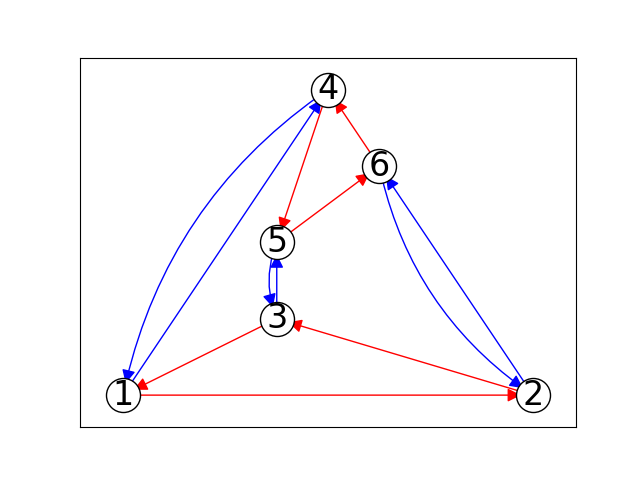}

Architecture 2.0

\includegraphics[width=0.5\textwidth]{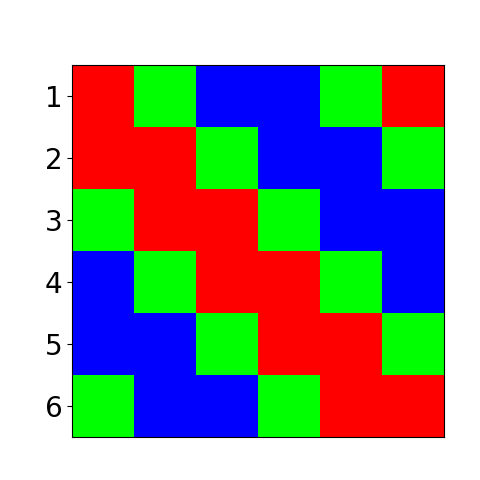}%
\includegraphics[width=0.5\textwidth]{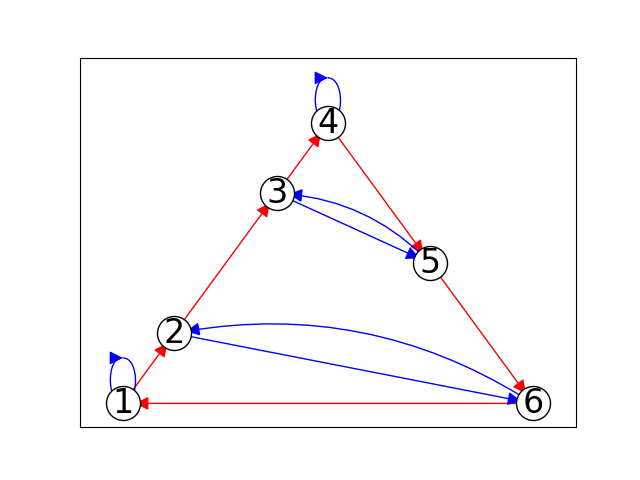}

Architecture 3.0

\includegraphics[width=0.5\textwidth]{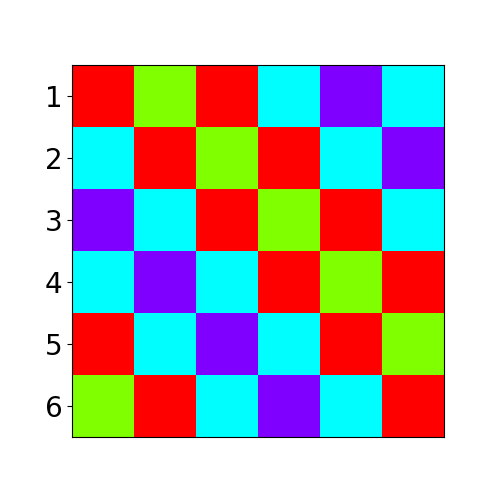}%
\includegraphics[width=0.5\textwidth]{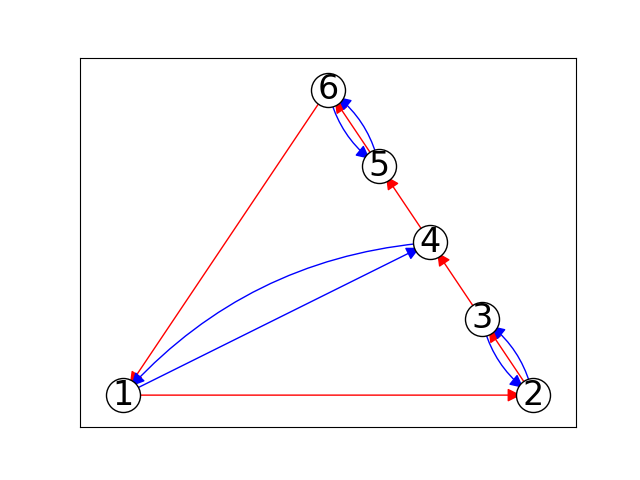}
\end{minipage}%
\begin{minipage}{0.5\textwidth}
\centering
Architecture 1.1

\includegraphics[width=0.5\textwidth]{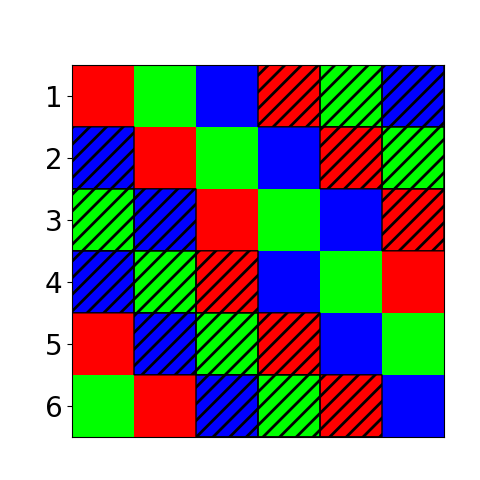}%
\includegraphics[width=0.5\textwidth]{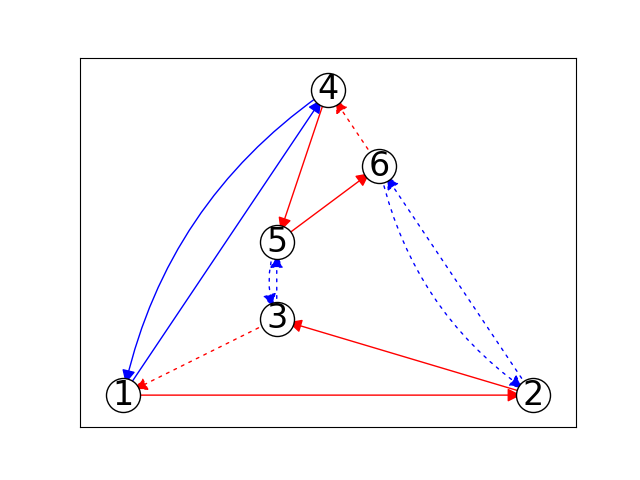}

Architecture 2.1

\includegraphics[width=0.5\textwidth]{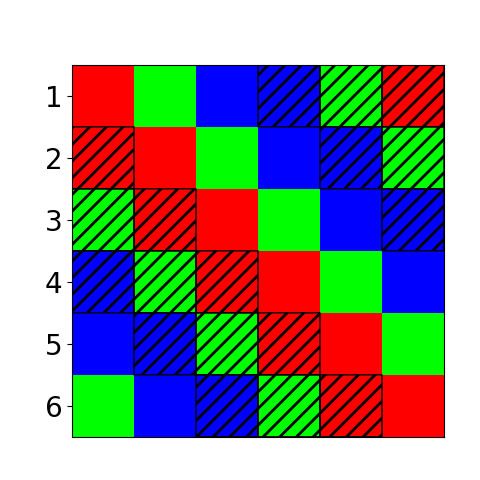}%
\includegraphics[width=0.5\textwidth]{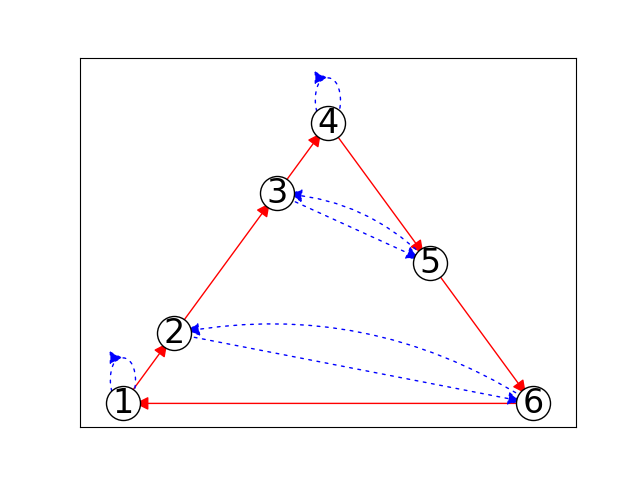}

Architecture 3.1

\includegraphics[width=0.5\textwidth]{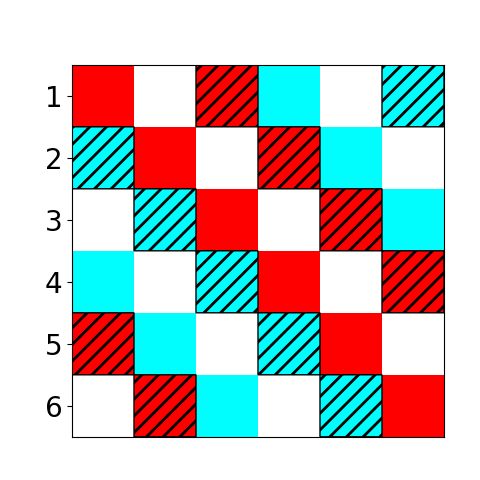}%
\includegraphics[width=0.5\textwidth]{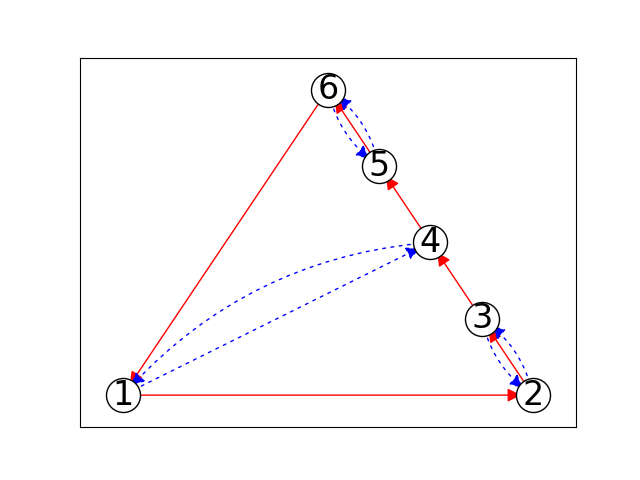}
\end{minipage}
\end{minipage}
\caption{\label{fig:D6_perm:2} %
Constraint patterns of the weight matrices and illustrations of the cohomology classes of the six-hidden-neuron irreducible $G$-SNN architectures for the dihedral permutation group $G=D_6$. 
Interpretation is the same as in Fig.~\ref{fig:C6_perm}. 
Weights colored white are constrained to equal zero. 
See Table~\ref{table:D6} to interpret the names ``architecture i.j''.
}
\end{figure}

\begin{figure}
\centering
\begin{minipage}{\textwidth}
\begin{minipage}{0.5\textwidth}
\centering
Architecture 4.0

\includegraphics[width=0.5\textwidth]{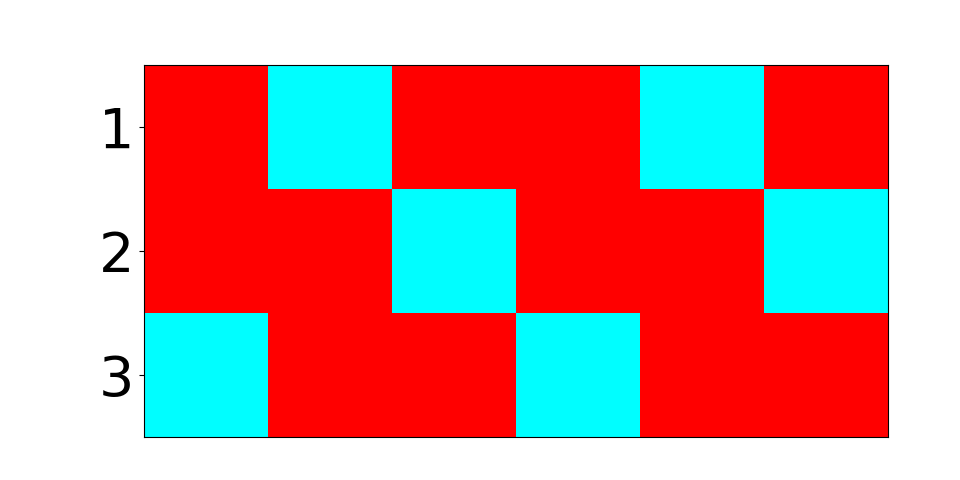}%
\includegraphics[width=0.5\textwidth]{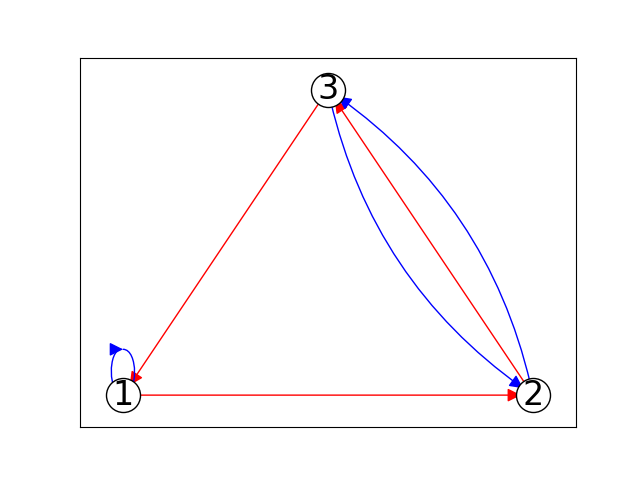}

Architecture 4.2

\includegraphics[width=0.5\textwidth]{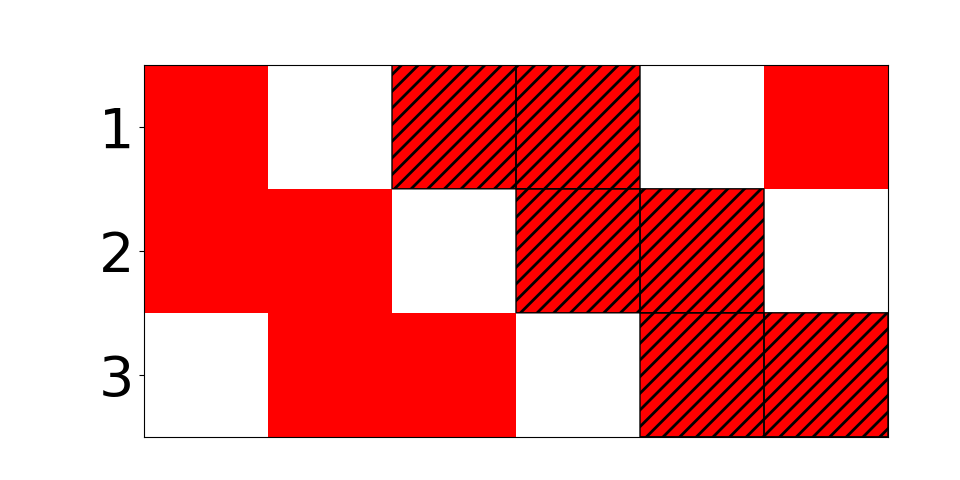}%
\includegraphics[width=0.5\textwidth]{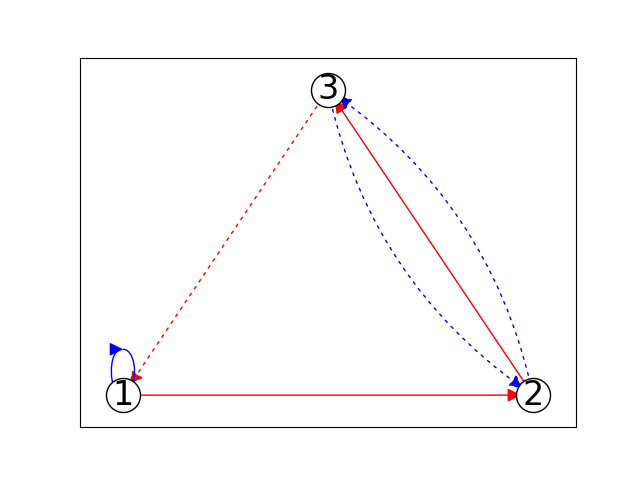}

Architecture 6.0

\includegraphics[width=0.5\textwidth]{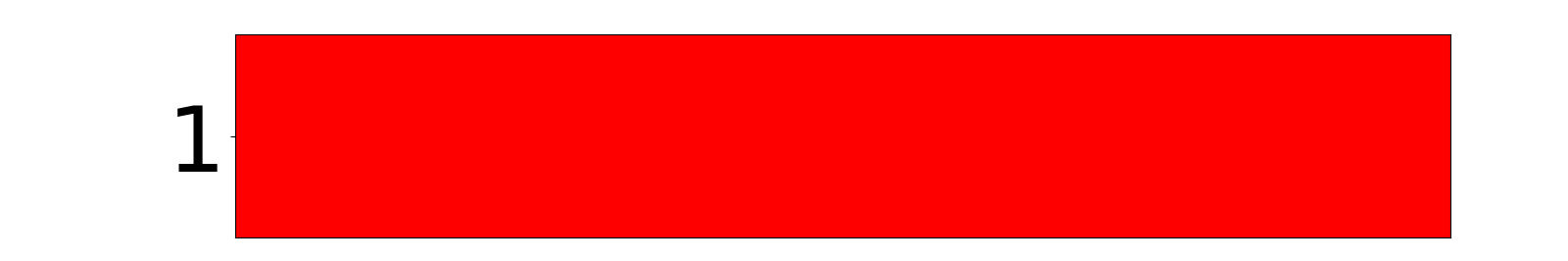}%
\includegraphics[width=0.5\textwidth]{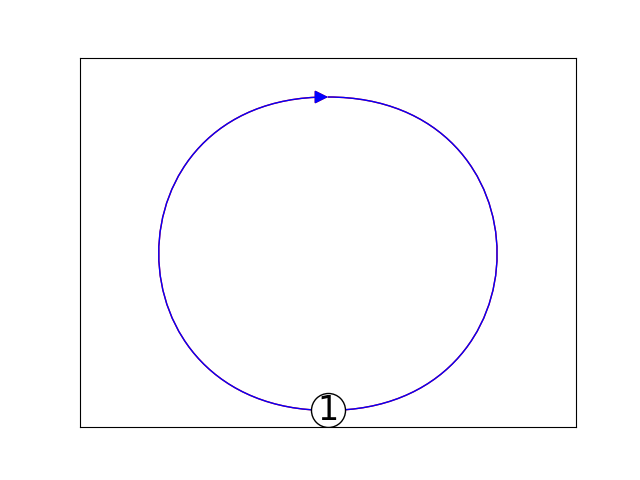}
\end{minipage}%
\begin{minipage}{0.5\textwidth}
\centering
Architecture 4.1

\includegraphics[width=0.5\textwidth]{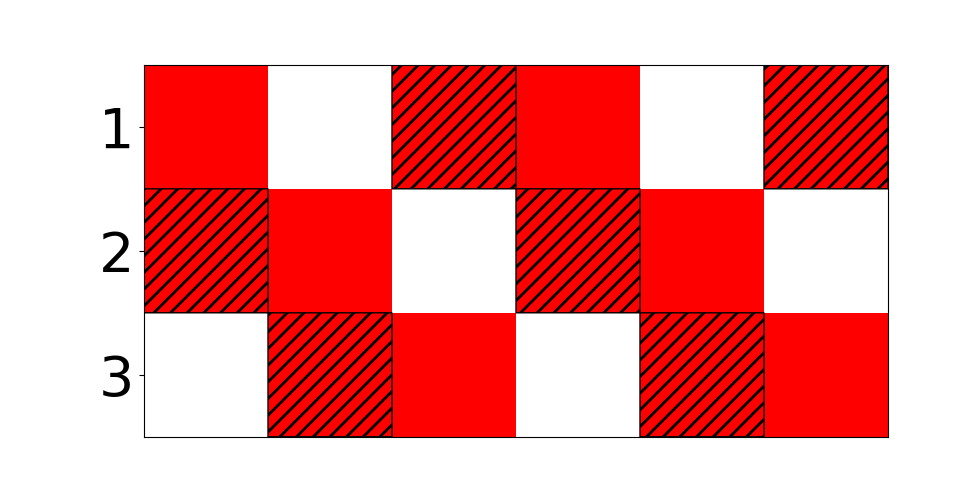}%
\includegraphics[width=0.5\textwidth]{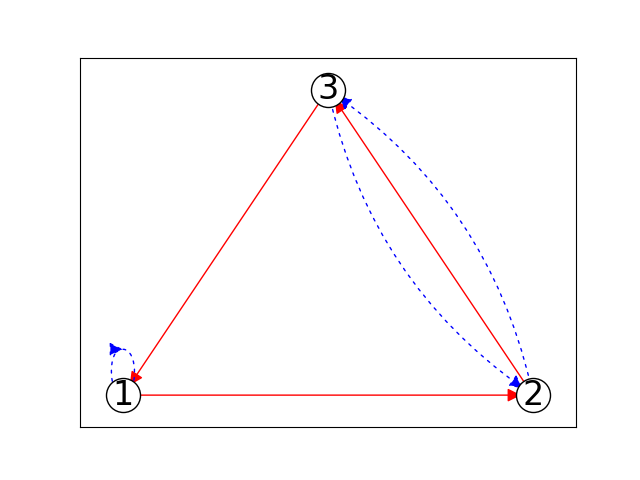}

Architecture 4.3

\includegraphics[width=0.5\textwidth]{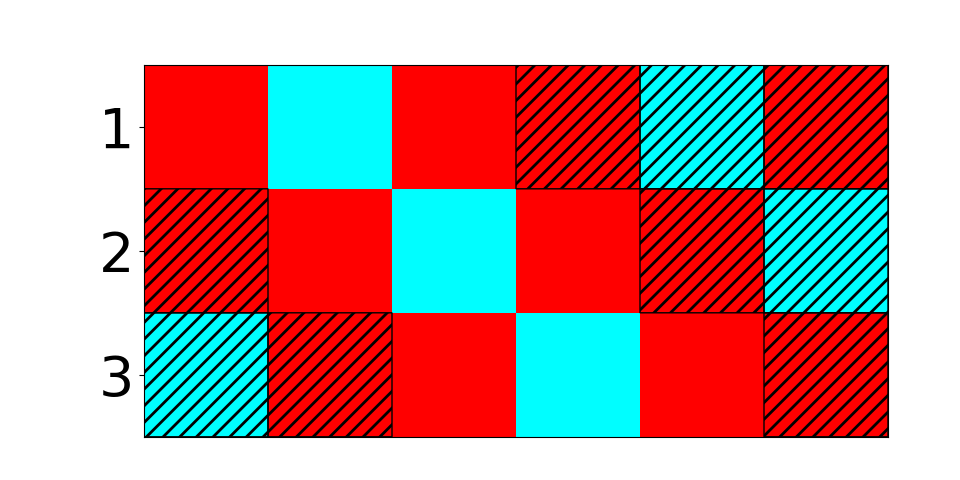}%
\includegraphics[width=0.5\textwidth]{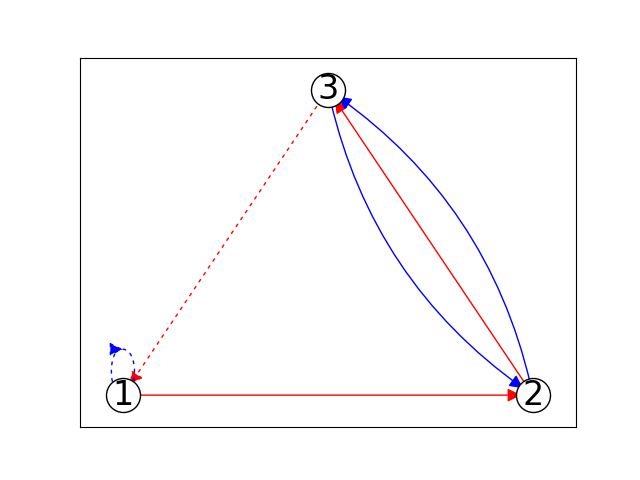}

Architecture 6.1

\includegraphics[width=0.5\textwidth]{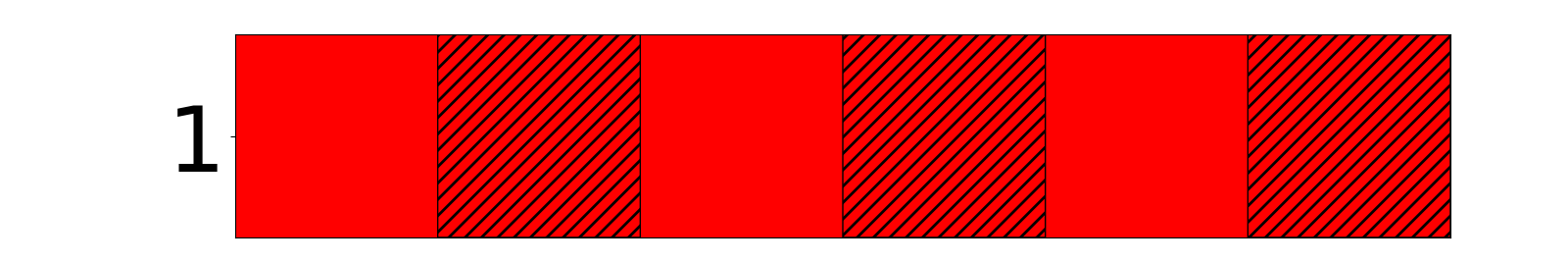}%
\includegraphics[width=0.5\textwidth]{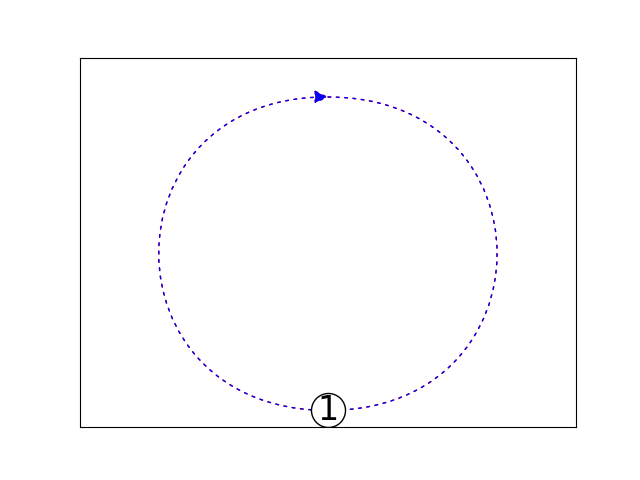}
\end{minipage}
\end{minipage}
\caption{\label{fig:D6_perm:3} %
Constraint patterns of the weight matrices and illustrations of the cohomology classes of the three-hidden-neuron and single-hidden-neuron irreducible $G$-SNN architectures for the dihedral permutation group $G=D_6$.
Interpretation is the same as in Fig.~\ref{fig:C6_perm}. 
See Table~\ref{table:D6} to interpret the names ``architecture i.j''.
}
\end{figure}

Consider the dihedral group $G=D_6$ of permutations generated by
\begin{align}
r &= [2, 3, 4, 5, 6, 1] \label{eq:r} \\
t &= [6, 5, 4, 3, 2, 1]. \label{eq:t}
\end{align}
There are 14 irreducible $G$-SNN architectures-- 7 of each type. 
We visualize their canonical weight matrices and corresponding cohomology classes in Figs.~\ref{fig:D6_perm:1}-\ref{fig:D6_perm:3}. 
Each architecture is named ``i.j'' where i and j index the subgroups $H$ and $K$ that are used to construct the architecture (see Thm.~\ref{thm:main}); 
Table~\ref{table:D6} lists these subgroups for each architecture.

In contrast to the cyclic permutation group (Sec.~\ref{sec:examples:perm}), the cohomology class illustrations for $G=D_6$ have arcs of two colors; 
red (resp. blue) arcs represent the action of the generator $r$ (resp. $t$). 
The existence of any loops with an odd number of dashed arcs indicates a nontrivial topology. 
For example, the four architectures 4.j with three hidden neurons correspond to the classes in $\mathcal{H}^1(G, M_{H_4})\cong C_2\times C_2$.

An important remark is that while there are only two architectures 6.j corresponding to the subgroup $H_6$, the corresponding cohomology group is $\mathcal{H}^1(G, M_{H_6})\cong C_2\times C_2$. 
Thus, there are two cohomology classes for which the corresponding signed perm reps failed the condition in Thm.~\myref{thm:main}{a}. 
It is thus not necessary for an irreducible architecture to exist for every cohomology class.

\begin{figure}
\centering
\includegraphics[width=0.8\textwidth]{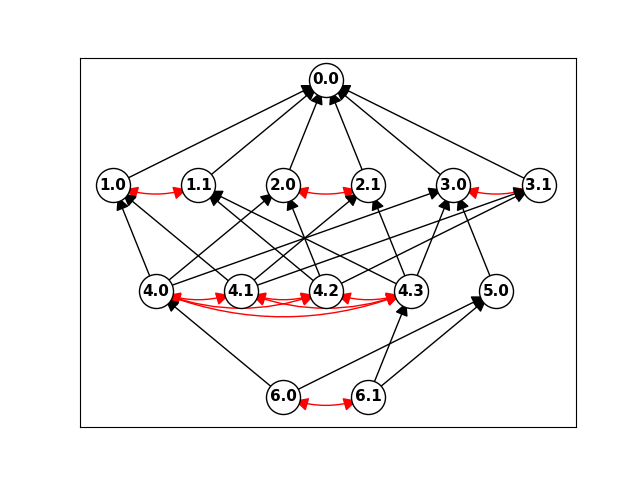}
\caption{\label{fig:D6_morphisms} %
Network morphisms between irreducible $G$-SNN architectures for the dihedral permutation group $G=D_6$. 
Every direct path in black represents an asymptotic inclusion. 
Red doubled-arrowed arcs represent the feasibility of topological tunneling.
}
\end{figure}

We also draw the network morphisms given by asymptotic inclusions between the irreducible architectures, as well as the shortcuts due to topological tunneling (Fig.~\ref{fig:D6_morphisms}); 
see Sec.~\ref{sec:remarks} for exposition on these concepts. 
As with Fig.~\ref{fig:C6_morphisms}, we determined the asymptotic inclusions manually by looking at the first rows of the weight matrices depicted in Figs.~\ref{fig:D6_perm:1}-\ref{fig:D6_perm:3} and observing how they nest. 
We obtain a $4$-partite topology on the architecture space, plus some topological ``tunnels'' between architectures belonging to a common cohomology ring.

\subsection{The dihedral rotation group}
\label{appendix:examples:rot}

\begin{figure}
\centering
\begin{tabular}{ccc}
\includegraphics[width=0.33\textwidth]{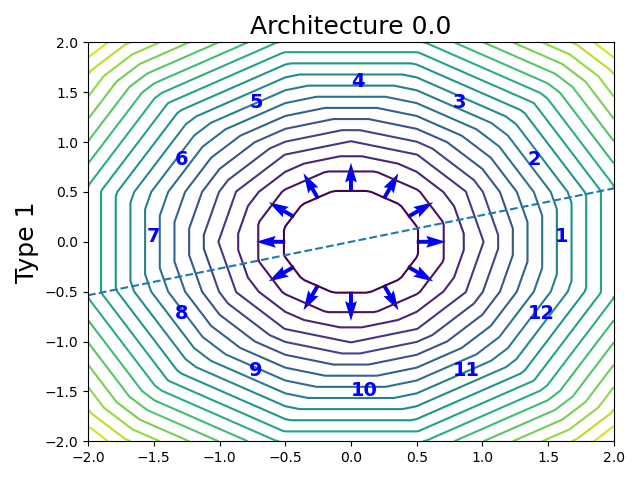} & %
\includegraphics[width=0.33\textwidth]{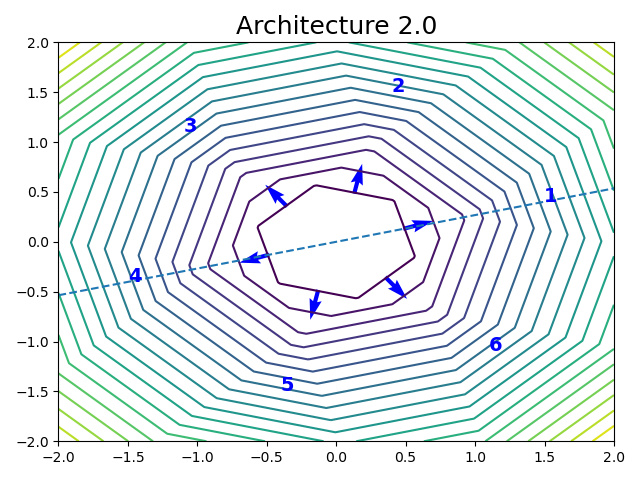} & %
\includegraphics[width=0.33\textwidth]{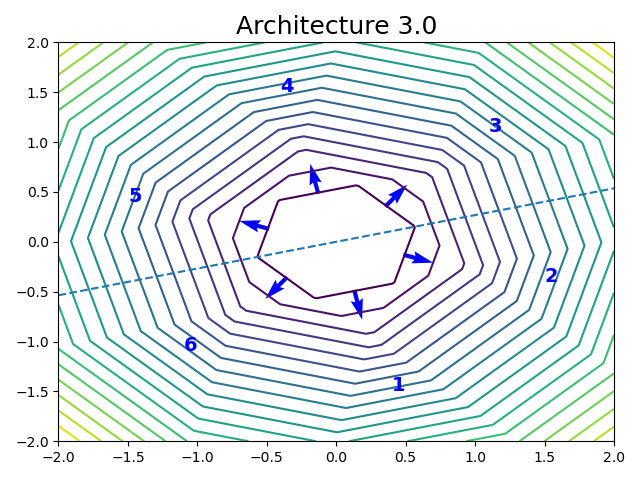} \\
\includegraphics[width=0.33\textwidth]{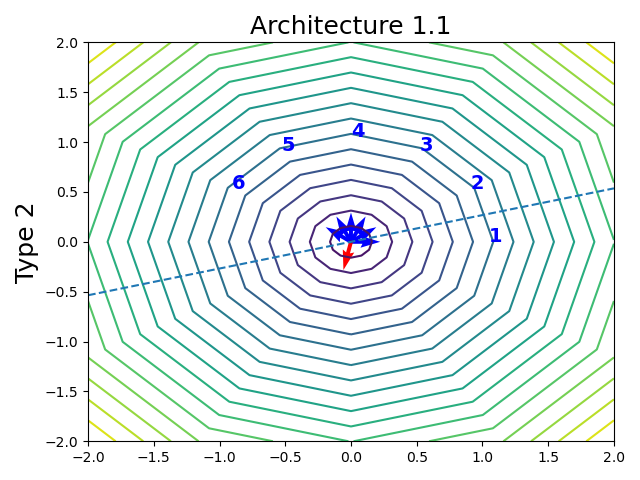} & %
\includegraphics[width=0.33\textwidth]{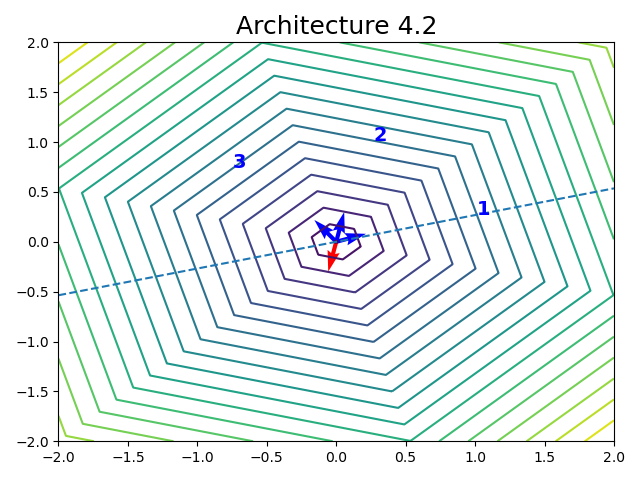} & %
\includegraphics[width=0.33\textwidth]{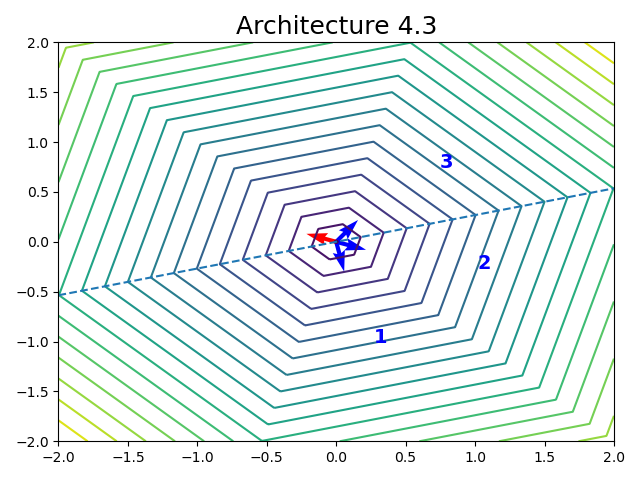}
\end{tabular}
\caption{\label{fig:D6_rot} %
Contour plots of the six irreducible $G$-SNN architectures for the 2D orthogonal representation of $G=D_6$. 
The architecture names are still based on Table~\ref{table:D6}, but the generator $r$ is now a $60^\circ$ rotation and $t$ is a reflection about the dashed black line. 
In architectures 0.0 and 1.1, weight vector 1 is unconstrained and was chosen arbitrarily for visualization. 
in the other four architectures, however, weight vector 1 is constrained to a 1D subspace as depicted. 
Interpretation of these plots is the same as in Fig.~\ref{fig:C6_rot}.
}
\end{figure}

As our final example, we consider a 2D orthogonal representation of the dihedral group $G=D_6$. 
In this representation, the generator $r$ is a $60^\circ$ counterclockwise rotation, and the generator $t$ is a reflection about the line $y=\tan(15^\circ)x$; 
we choose this line of reflection solely because it makes the example interesting. 
There are six irreducible $G$-SNN architectures for this group---three of each type---and we visualize there contour plots (Fig.~\ref{fig:D6_rot}). 
The level curves are clearly invariant under $60^\circ$ rotations and are symmetric about $y=\tan(15^\circ)x$. 
The architecture names are still based on Table~\ref{table:D6}, but the generators $r$ and $t$ are now 2D orthogonal transformations instead of permutations. 
Note that for architectures 0.0 and 1.1, the corresponding subgroup $K$ appearing in Thm.~\ref{thm:main} is $K=\brak{e}$; 
for architectures 2.0 and 4.2, $K=\brak{t}$; 
and for architectures 3.0 and 4.3, $K=\brak{r^3t}$, 
whence the columns in Fig.~\ref{fig:D6_rot}.

Each type 2 architecture in the second row of Fig.~\ref{fig:D6_rot} is asymptotically included in the type 1 architecture depicted above it. 
This is the same situation as with the cyclic rotation group in Sec.~\ref{sec:examples:rot}; 
each type 1 architecture approaches the corresponding type 2 architecture below it as its bias parameter $b_*$ tends to zero. 
In contrast to the cyclic rotation group, however, there are other architectures that are effectively confined from one another. 
The weight vectors of architecture 2.0 are orthogonal to those of architecture 3.0, 
and hence one architecture can reach the other only if it passes through a degenerate network with zero-valued weight vectors. 
By the same token, architectures 4.2 and 4.3 are topologically confined from one another, just as Prop.~\ref{prop:orthogonal} states.

Finally, we note that while we have architecture 1.1, there is no architecture 1.0; 
similarly, we have architectures 4.2 and 4.3 but not 4.0 and 4.1. 
This example thus demonstrates that the cohomology classes in a given cohomology ring for which the corresponding $G$-SNN architecture exists (i.e., satisfies the condition in Thm.~\myref{thm:main}{a}) need not form a subgroup, 
and this raises the question: 
Are there any discernible patterns in the set of irreducible $G$-SNN architectures (which satisfy Thm.~\myref{thm:main}{a}) as we vary $G$? 
We will investigate this further as part of future work.

\section{Remarks}
\label{appendix:remarks}

\subsection{Asymptotic inclusion}
\label{appendix:ai}

Let $\SNNirr(G)$ be the space of all irreducible $G$-SNNs equipped with the topology of uniform convergence on compact sets. 
Let $S^{m-1}_+$ denote a hemisphere of the $(m-1)$-dimensional unit sphere such that it has no antipodal pairs of points. 
Choose an ordering on $G$ so it has elements $g_1,\ldots,g_{|G|}$ with $g_1=1$, 
and define the map $\phi:\RR^{\neq 0}\times S^{m-1}_+\times \RR\times \ran(P_G)\times\RR\mapsto \SNNirr(G)$ by
\[ [\phi(a, w, b, c, d)](x) = a\vec{1}_{|G|}^{\top}\relu(Wx+b\vec{1}_{|G|}) + c^{\top}x + d\forall x\in \RR^m, \]
where the $|G|\times m$ weight matrix $W$ is defined as
\[ W = \sum_{i=1}^{|G|} e_i (g_i w)^{\top}. \]
We call this the \emph{unraveled parameterization} of $G$-SNNs, and it has the advantage that it has $|G|$ hidden neurons regardless of the associated signed perm-rep. 
We can easily transform it into the canonical parameterization as follows: 
Let $K\leq H\leq G$ be the largest subgroups such that $|H:K|\leq 2$ and $w\in P_K-\tau P_H$ where $\tau=|H:K|-1$; 
replace the parameter $a$ with $a|H|$; 
and finally, use $(a|H|, w, b, c, d)$ and Thm.~\myref{thm:main}{b} to construct the canonical form. 
This is ``rolling up'' the $G$-SNN so that only $|G|/|H|$ of the $|G|$ hidden neurons remain. 
Note that this procedure can be reversed, so that we can move back-and-forth between the unraveled and canonical parameterizations. 
As a consequence, it immediately follows that $\phi$ is a well-defined function, in the sense that it outputs a $G$-SNN that is indeed \emph{irreducible}, and is surjective.

Let $\Omega$ be a fundamental domain in $S^{m-1}_+$ under the action of $G$ 
(note that $S^{m-1}_+\subset \RR^m$), 
and let $\phi\mid_{\Omega}$ denote the restriction of $\phi$ to $\RR^{\neq 0}\times\Omega\times\RR\times\ran(P_G)\times\RR$. 
Then $\phi\mid_{\Omega}$ is injective as well and thus a bijection; 
indeed, without this restriction, $w$ and $gw$ for any $g\in G$ would both generate the same weight matrix $W$ in the unraveled parameterization up to the order of its rows, 
and the restriction to a fundamental domain breaks this redundancy. 

The following lemma will help us prove Thm.~\ref{thm:ai}.

\begin{lemma} \label{lemma:ai}
We have:
\begin{enumerate}[label={(\alph*)}]
\item \label{lemma:ai:a}
$\phi$ is a continuous function. 
\item \label{lemma:ai:b}
Let $\{f_n\in\SNNirr(G)\}_{n=1}^{\infty}$ be a sequence such that $f_n\rightarrow f\in\SNNirr(G)$ in the topology of $\SNNirr(G)$. 
Let $(a_n, w_n, b_n, c_n, d_n)=\left(\phi\mid_{\Omega}\right)^{-1}(f_n)$ for each $n$ and similar for $f$. 
Then there exists $g\in G$ such that $(gw_n, b_n)\rightarrow \pm (w, b)$.
\end{enumerate}
\end{lemma}
\begin{proof}
\textbf{(a) } %
Let $\{\theta_n\in\dom(\phi)\}_{n=1}^{\infty}$ be a convergent sequence with limit $\theta\in\dom(\phi)$. 
For each $n$, let $f_n = \phi(\theta_n)$, and let $f=\phi(\theta)$. 
For each $x\in\RR^m$, the function $\phi(\cdot)(x):\dom(\phi)\mapsto\RR$ is continuous 
and hence $f_n\rightarrow f$ pointwise over the entire domain $\RR^m$.

Now since $G$-SNNs are piecewise-linear functions and $\{\theta_n\}_{n=1}^{\infty}$ has a finite limit $\theta$, 
then clearly the derivatives of the $f_n$ are uniformly bounded, 
so that in particular $\{f_n\}_{n=1}^{\infty}$ is a sequence of  equicontinuous functions. 
By the Arzel\`{a}-Ascoli Theorem, $f_n\rightarrow f$ uniformly on every compact set; 
i.e., $f_n\rightarrow f$ in the topology of $\SNNirr(G)$, 
thus establishing the continuity of $\phi$.

\textbf{(b) } %
Since $a\neq 0$, then $f$ is nonlinear. 
For $f_n$ to converge to a nonlinear function, at least one hyperplane on which $f_n$ is non-differentiable must converge to a hyperplane on which $f$ is non-differentiable. 
Thus, there exist $g_1,g_2\in G$ such that $(g_1w_n, b_n)\rightarrow \pm (g_2w, b)$. 
Equivalently, there exists $g\in G$ such that $(gw_n, b_n)\rightarrow \pm (w, b)$.
\end{proof}

We now prove Thm.~\ref{thm:ai}.

\begin{proof}[Proof of Thm.~\ref{thm:ai}] 
For the forward implication, suppose $f^{\PZ}_2\hookrightarrow f^{\PZ}_1$. 
Let $w\in\ran(P_{K_2}-\tau_2 P_{H_2})$ where $\tau_2=|H_2:K_2|-1$. 
Then there exists $g_2\in G$ such that $\pm g_2w\in\ran(P_{K_2}-\tau_2 P_{H_2})\cap\Omega$; 
without loss of generality, we assume $g_2 w\in\ran(P_{K_2}-\tau_2 P_{H_2})\cap\Omega$. 
Now, there exists $f\in f_2^{\PZ}$ with top weight vector $g_2 w$. 
Since $f_2^{\PZ}\hookrightarrow f_1^{\PZ}$, then there exists $\{f_n\in f_1^{\PZ}\}_{n=1}^{\infty}$ such that $f_n\rightarrow f$ in the topology of $\SNNirr(G)$. 
By Lemma~\myref{lemma:ai}{b}, there exists $g_1\in G$ such that $(g_1w_n, b_n)\rightarrow \pm (g_2w, b)$, 
where $w_n$ and $b_n$ are the weight and bias parameters of $f_n$ and $b$ is the bias parameter of $f$ respectively. 
Thus, there exists $g\in G$ such that $\pm gw_n\rightarrow w$. 
Since $w_n\in\ran(P_{K_1}-\tau_1 P_{H_1})$ where $\tau_1=|H_1:K_1|-1$, then $\pm gw_n\in\ran(gP_{H_1}g^{-1}-\tau_1 gP_{H_1}g^{-1})$. 
Letting $(H, K) = g(H_1, K_1)g^{-1}$, we have $w_n\in\ran(P_K-\tau P_H)$ where $\tau=|H:K|-1$. 
We thus establish that $\ran(P_{H_2}-\tau_2 P_{H_2})\subseteq\ran(P_K-\tau P_H)$.

The space $\ran(P_{K_2}-\tau_2 P_{H_2})$ is thus in particular fixed pointwise by every element of $K$, so that $K\leq \st_G(P_{K_2}-\tau_2 P_{H_2})$. 
By Thm.~\myref{thm:main}{a}, however, $\st_G(P_{K_2}-\tau_2 P_{H_2}) = K_2$, so that $K\leq K_2$. 
In the case $f_1^{\PZ}$ is type 1, we have $H=K\leq K_2\leq H_2$ and thus $H\cap K_2=K$, and hence we are done. 
Suppose instead $f_1^{\PZ}$ is type 2. 
Then $f_2^{\PZ}$ must be type 2 as well; 
if it were type 1, then we could set its bias parameter $b$ to be nonzero, and $f_1^{\PZ}$ would be unable to reach it asymptotically as its own bias is constrained to $b=0$. 
With both $f_1^{\PZ}$ and $f_2^{\PZ}$ type 2, we have $\ran(P_{K_2}-P_{H_2})\subseteq\ran(P_K-P_H)$. 
In particular, for any $h\in H\setminus K$, we must have
\[ h(P_{K_2}-P_{H_2}) = -(P_{K_2}-P_{H_2}). \]
However, for the rows of the canonical weight matrix of any $f\in f_2^{\PZ}$ to be pairwise nonparallel, we must have $g(P_{K_2}-P_{H_2}) = -(P_{K_2}-P_{H_2})$ implies $g\in H_2$. 
Hence, $H\setminus K\subseteq H_2$. 
Combining this with $K\leq K_2<H_2$, we obtain $H\leq H_2$. 
Finally, since $\ran(P_{K_2}-P_{H_2})$ must be fixed under each element of $K_2$ but \emph{not} fixed under each element of $H\setminus K$, then we must have $(H\setminus K)\cap K_2=\emptyset$, from which we conclude $H\cap K_2=K$.

For the reverse implication, suppose there exists $(H, K)\in (H_1, K_1)^G$ such that $H\leq H_2$, $K\leq K_2$, and $H\cap K_2=K$. 
Without loss of generality, let $(H, K) = (H_1, K_1)$. 
Let $f\in f_2^{\PZ}$ an $(a, w, b, c, d)=\left(\phi\mid_{\Omega}\right)^{-1}(f)$. 
Define the sequence $\{f_n\in f_1^{\PZ}\}_{n=1}^{\infty}$ and $(a_n, w_n, b_n, c_n, d_n)=\left(\phi\mid_{\Omega}\right)^{-1}(f_n)$, 
where we set $a_n=a$, $c_n=c$, and $d_n=d$ for all $n$. 
We want to show the existence of $w_n$ and $b_n$ such that $w_n\rightarrow w$ and $b_n\rightarrow b$; 
Lemma~\myref{lemma:ai}{a} will then give us the desired result.

Suppose $f_2^{\PZ}$ is type 1. 
Since $K_1\leq K_2=H_2$ and $H_1\leq H_2$, then $K_1\leq H_1\leq K_2$ and hence $H_1\cap K_2=H_1$. 
On the other hand, since $H_1\cap K_2=K_1$, then $H_1=K_1$ so that $f_1^{\PZ}$ is type 1 as well. 
In this case, we set $b_n=b$ for all $n$, and we have $w_n\in\ran(P_{K_1})$ and $w\in\ran(P_{K_2})\subseteq\ran(P_{K_1})$, 
thus establishing the existence of a sequence $w_n\rightarrow w$. 
On the other hand, suppose $f_2^{\PZ}$ is type 2. 
Then $b=0$, and we set $b_n=0$ for all $n$ if $f_1^{\PZ}$ is type 2 or $b_n\rightarrow 0$ if $f_1^{\PZ}$ is type 1. 
From the hypotheses, it is easy to verify that
\[ \ran(P_{K_2}-P_{H_2})\subseteq \ran(P_{K_1}-P_{H_1})\subseteq \ran(P_{K_1}). \]
It follows that regardless of the type of $f_1^{\PZ}$, a sequence $w_n\rightarrow w$ exists, thereby establishing the claim.
\end{proof}

\subsection{Topological confinement}
\label{appendix:topological}

We prove Prop.~\ref{prop:orthogonal}, which states that non-cohomologous irreducible $G$-SNN architectures are in a sense orthogonal.

\begin{proof}[Proof of Prop.~\ref{prop:orthogonal}] %
For $i=1,2$, let $\tau_i = |H:K_i|$. 
Then by Eq.~\ref{eq:main:W}, we have the constraints $w_i\in\ran(P_{K_i}-\tau_iP_H)$. 
If one of the $K_i$ equals $H$, say $K_1=H$, then in particular we have the constraints $P_Hw_1 = w_1$ and $(P_{K_2}-P_H)w_2 = w_2$. 
We thus have
\begin{align*}
w_1^{\top}w_2 
&= w_1^{\top}P_H^{\top}(P_{K_2}-P_H)w_2 \\
&= w_1^{\top}P_H(P_{K_2}-P_H)w_2 \\
&= w_1^{\top}(P_HP_{K_2}-P_H)w_2 \\
&= 0,
\end{align*}
where the last step holds because $K_2\leq H$ and hence $P_HP_{K_2}=P_H$.

On the other hand, suppose $K_1$ and $K_2$ are both proper subgroups of $H$. 
Since $P_{K_i}w_i = w_i$, then
\[ w_1^{\top}w_2 = w_1^{\top}P_{K_1}^{\top}P_{K_2}w_2. \]
Let $K = K_1\cap K_2$. 
It is well-known that because $K_1$ and $K_2$ are distinct index-2 subgroups of $H$, then $K$ is an index-4 subgroup of $H$ and
\[ H/K = \{K, k_1K, k_2K, hK\}, \]
where $k_i\in K_i\setminus K$ for $i=1,2$ and $h\in H\setminus (K_1\cup K_2)$. 
We thus have
\begin{align*}
P_{K_1}^{\top}P_{K_2} 
&= \left(\frac{1}{|K_1|}\sum_{k\in K_1} k\right)^{\top}\left(\frac{1}{|K_2|}\sum_{k\in K_2}k\right) \\
&= \frac{1}{|K_1|^2}\left(\sum_{k\in K}k + k_1\sum_{k\in K}k\right)^{\top}\left(\sum_{k\in K}k + k_2\sum_{k\in K}k\right) \\
&= \frac{|K|^2}{|K_1|^2}(P_K+k_1P_K)^{\top}(P_K+k_2P_K) \\
&= \frac{1}{4}P_K^{\top}(I + k_1^{\top})(I + k_2)P_K
&= \frac{1}{4}P_K^{\top}(I + k_1^{\top} + k_2 + k_1^{\top}k_2)P_K.
\end{align*}
Since $P_Kw_i = w_i$ for both $i=1,2$, then we have
\begin{align*}
w_1^{\top}w_2 
&= \frac{1}{4}w_1^{\top}P_K^{\top}(I + k_1^{\top} + k_2 + k_1^{\top}k_2)P_Kw_2 \\
&= \frac{1}{4}w_1^{\top}(I + k_1^{\top} + k_2 + k_1^{\top}k_2)w_2.
\end{align*}
Since $k_1\in K_1$, then $k_1w_1 = w_1$, and hence $w_1^{\top}k_1^{\top}w_2 = w_1^{\top}w_2$. 
On the other hand, since $k_1^{\top}\in H\setminus K_2$, then $k_1^{\top}w_2 = -w_2$ so that $w_1^{\top}k_1^{\top}w_2 = -w_1^{\top}w_2$, 
thus implying $w_1^{\top}k_1^{\top}w_2 = 0$. 
We can similarly show that $w_1^{\top}k_2w_2 = 0$. 
This leaves
\begin{align*}
w_1^{\top}w_2 
&= \frac{1}{4}w_1^{\top}(I + k_1^{\top}k_2)w_2 \\
&= \frac{1}{4}(w_1^{\top}w_2 + w_1^{\top}w_2) \\
&= \frac{1}{2}w_1^{\top}w_2,
\end{align*}
implying $w_1^{\top}w_2 = 0$.
\end{proof}

\end{document}